\documentclass[final,12pt]{colt2021} 

\title[Last-iterate Convergence of OGDA in Infinite-horizon Markov Games]{Last-iterate Convergence of Decentralized Optimistic Gradient Descent/Ascent in Infinite-horizon Competitive Markov Games}
\usepackage{times} 

 \coltauthor{\Name{Chen-Yu Wei} \Email{chenyu.wei@usc.edu}\\
  \Name{Chung-Wei Lee\nametag{\thanks{Equal contribution.}}} \Email{leechung@usc.edu}\\
  \Name{Mengxiao Zhang\nametag{\footnotemark[1]}} \Email{mengxiao.zhang@usc.edu}\\
  \Name{Haipeng Luo} \Email{haipengl@usc.edu}\\
  \addr University of Southern California}
  
\usepackage{xspace}
\usepackage{amsmath}
\usepackage{bbm}
\usepackage{algorithm}
\usepackage{algorithmic}
\usepackage{mathrsfs}
\usepackage{amssymb}
\usepackage{bm}
\usepackage{makecell}
\usepackage{tabulary}
\usepackage{xcolor}
\usepackage{prettyref}
\usepackage{mathtools}
\usepackage{amsmath}
\usepackage{makecell}
\newcommand{\LL}{L}

\definecolor{Green}{rgb}{0.13, 0.65, 0.3}

\DeclareMathOperator*{\argmin}{argmin} 
\newcommand{\pconst}{\frac{2\eta}{(1-\gamma)^2}}
\newcommand{\rconst}{2\eta}

\newcommand{\sig}{\sigma}

\newcommand{\pay}{\rho}
\newcommand{\err}{\varepsilon}
\newcommand{\arx}{\xi}
\newcommand{\calT}{\mathcal{T}}

\newcommand{\calA}{\mathcal{A}}
\newcommand{\calB}{\mathcal{B}}

\newcommand{\calZ}{\mathcal{Z}}

\newcommand{\calS}{\mathcal{S}}

\newcommand{\gametwo}{\underline}

\newcommand{\diff}{\text{diff}}

\newcommand{\barReg}{\overline{\text{\rm Reg}}}
\newcommand{\emphdef}[1]{{\color{red}\emph{(#1)}}\xspace}
\newcommand{\term}{\textbf{term}}
\newcommand{\tildeq}{\widetilde{q}}
\newcommand{\tildep}{\widetilde{p}}
\newcommand{\tildex}{\widetilde{x}}
\newcommand{\tildey}{\widetilde{y}}
\newcommand{\tildez}{\widetilde{z}}

\newcommand{\E}{\mathbb{E}}
\newcommand{\order}{\mathcal{O}}

\newcommand{\norm}[1]{\left\|#1\right\|}

\newcommand{\calX}{\mathcal{X}}
\newcommand{\calY}{\mathcal{Y}}
\newcommand{\one}{\mathbbm{1}}
\newcommand{\inner}[1]{\langle#1\rangle}

\newcommand{\dist}{\mathrm{dist}}

\newcommand{\xtil}{\widetilde{x}}
\newcommand{\ytil}{\widetilde{y}}

\newcommand{\difz}{J}
\newcommand{\difq}{K}

\newcommand{\gap}{\Gamma}

\newtheorem{assumption}{Assumption}

\newcommand{\x}{x}
\newcommand{\y}{y}
\newcommand{\z}{z}
\newcommand{\xp}{\widehat{x}}
\newcommand{\yp}{\widehat{y}}

\newcommand{\zp}{\widehat{z}}

\newcommand{\Q}{Q}
\newcommand{\constB}{C_\beta}
\newcommand{\constA}{C_\alpha}
\usepackage{nicefrac}

\newcommand{\pref}[1]{\prettyref{#1}}

\newcommand{\savehyperref}[2]{\texorpdfstring{\hyperref[#1]{#2}}{#2}}
\newrefformat{eq}{\savehyperref{#1}{Eq.~\textup{(\ref*{#1})}}}
\newrefformat{eqn}{\savehyperref{#1}{Equation~\ref*{#1}}}
\newrefformat{lem}{\savehyperref{#1}{Lemma~\ref*{#1}}}
\newrefformat{lemma}{\savehyperref{#1}{Lemma~\ref*{#1}}}
\newrefformat{def}{\savehyperref{#1}{Definition~\ref*{#1}}}
\newrefformat{line}{\savehyperref{#1}{Line~\ref*{#1}}}
\newrefformat{thm}{\savehyperref{#1}{Theorem~\ref*{#1}}}
\newrefformat{theorem}{\savehyperref{#1}{Theorem~\ref*{#1}}}
\newrefformat{corr}{\savehyperref{#1}{Corollary~\ref*{#1}}}
\newrefformat{cor}{\savehyperref{#1}{Corollary~\ref*{#1}}}
\newrefformat{col}{\savehyperref{#1}{Corollary~\ref*{#1}}}
\newrefformat{sec}{\savehyperref{#1}{Section~\ref*{#1}}}
\newrefformat{subsec}{\savehyperref{#1}{Section~\ref*{#1}}}
\newrefformat{app}{\savehyperref{#1}{Appendix~\ref*{#1}}}
\newrefformat{assum}{\savehyperref{#1}{Assumption~\ref*{#1}}}
\newrefformat{ex}{\savehyperref{#1}{Example~\ref*{#1}}}
\newrefformat{fig}{\savehyperref{#1}{Figure~\ref*{#1}}}
\newrefformat{alg}{\savehyperref{#1}{Algorithm~\ref*{#1}}}
\newrefformat{algo}{\savehyperref{#1}{Algorithm~\ref*{#1}}}
\newrefformat{rem}{\savehyperref{#1}{Remark~\ref*{#1}}}
\newrefformat{conj}{\savehyperref{#1}{Conjecture~\ref*{#1}}}
\newrefformat{prop}{\savehyperref{#1}{Proposition~\ref*{#1}}}
\newrefformat{proto}{\savehyperref{#1}{Protocol~\ref*{#1}}}
\newrefformat{prob}{\savehyperref{#1}{Problem~\ref*{#1}}}
\newrefformat{claim}{\savehyperref{#1}{Claim~\ref*{#1}}}
\newrefformat{que}{\savehyperref{#1}{Question~\ref*{#1}}}
\newrefformat{op}{\savehyperref{#1}{Open Problem~\ref*{#1}}}
\newrefformat{fn}{\savehyperref{#1}{Footnote~\ref*{#1}}}
\newrefformat{tab}{\savehyperref{#1}{Table~\ref*{#1}}}
\newrefformat{fig}{\savehyperref{#1}{Figure~\ref*{#1}}}

\begin{document}

\maketitle

\begin{abstract}%
    We study infinite-horizon discounted two-player zero-sum Markov games, and develop a decentralized algorithm that provably converges to the set of Nash equilibria under self-play. Our algorithm is based on running an Optimistic Gradient Descent Ascent algorithm on each state to learn the policies, with a critic that slowly learns the value of each state. To the best of our knowledge, this is the first algorithm in this setting that is simultaneously {\it rational} (converging to the opponent's best response when it uses a stationary policy), {\it convergent} (converging to the set of Nash equilibria under self-play), {\it agnostic} (no need to know the actions played by the opponent), {\it symmetric} (players taking symmetric roles in the algorithm), and enjoying a {\it finite-time last-iterate convergence} guarantee, all of which are desirable properties of decentralized algorithms.  
\end{abstract}


\section{Introduction}

Multi-agent reinforcement learning studies how multiple agents should interact with each other and the environment, and has wide applications in, for example, playing board games \citep{silver2017mastering} and real-time strategy games \citep{vinyals2019grandmaster}.
To model these problems, the framework of Markov games (also called stochastic games) \citep{shapley1953stochastic} is often used, which can be seen as a generalization of Markov Decision Processes (MDPs) from a single agent to multiple agents.
In this work, we focus on one fundamental class: two-player zero-sum Markov games.

In this setting, there are many \emph{centralized} algorithms developed in a line of recent works with near-optimal sample complexity for finding a Nash equilibrium~\citep{wei2017online, sidford2020solving, xie2020, bai2020provable, zhang2020model, liu2020sharp}.
These algorithms require a central controller that collects some global knowledge (such as the actions and the rewards of all players) and then jointly decides the policies for all players.
Centralized algorithms are usually \emph{convergent} (as defined in~\citep{bowling2001rational}), in the sense that the policies of the players converge to the set of Nash equilibria.

On the other hand, there is also a surge of studies on \emph{decentralized} algorithms that run independently on each player, requiring only local information such as the player's own action and the corresponding reward feedback~\citep{zhang2019policy, bai2020near, tian2020provably, liu2021policy, daskalakis2021independent}.
Compared to centralized ones, decentralized algorithms are usually more versatile and can potentially run in different environments (cooperative or competitive).
Many of them enjoy the property of being \emph{rational} (as defined in~\citep{bowling2001rational}), in the sense that a player's policy converges to the best response to the opponent no matter what stationary policy the opponent uses.
However, it is also often more challenging to show the convergence to a Nash equilibrium when the two players execute the same decentralized algorithm. 

It can be seen that a rational algorithm has different benefits compared to a convergent algorithm -- the former satisfies individual player's interests, while the latter might be better for achieving social good. Therefore, a single algorithm that possesses both properties is highly desirable. For example, in a market where “enforcing” all traders to follow the same rule is difficult, but “recommending” them to use a specific algorithm is possible, a rational and convergent algorithm would be a good candidate --- if all traders follow the recommendation, then a social equilibrium is quickly attained; otherwise, those who follow the recommendation are still satisfied because they best respond to a stationary environment. 

Based on this motivation, our main contribution is to develop the first decentralized algorithm that is simultaneously \emph{rational}, \emph{last-iterate convergent} (with a concrete finite-time guarantee),\footnote{Note that while average-iterate convergence is possible (and standard) for stateless convex-concave games (e.g. \citep{syrgkanis2015fast}), it does not work for Markov games since the problem is nonconvex-nonconcave in the space of policies \citep{daskalakis2021independent}.} \emph{agnostic}, and \emph{symmetric} (more details to follow in \pref{subsec: related work}) for two-player zero-sum Markov games. 
Our algorithm is based on Optimistic Gradient Descent/Ascent (OGDA)~\citep{chiang2012online, rakhlin2013optimization} and importantly relies on a critic that slowly learns a certain value function for each state.
Following previous works on learning MDPs~\citep{abbasi2019politex, agarwal2019theory} or Markov games~\citep{perolat2018actor},
we present the convergence guarantee in terms of the number of iterations of the algorithm and the estimation error of some gradient information (along with other problem-dependent constants),
where the estimation error can be zero in a full-information setting, or goes down to zero fast enough with additional structural assumptions (e.g. every stationary policy pair induces an irreducible Markov chain, similar to~\citep{ortner2007logarithmic}).

While the OGDA algorithm, first studied in \citep{popov1980modification} under a different name, has been extensively used in recent years for learning matrix games (a special case of Markov games with one state), to the best of our knowledge, no previous work has applied it to learning Markov games and derived a concrete last-iterate convergence rate.
Several recent works derive last-iterate convergence of OGDA for matrix games~\citep{hsieh2019convergence, liang2019interaction, mokhtari2019unified, golowich2020tight, wei2020linear}, 
and our analysis is heavily inspired by the approach of~\citep{wei2020linear}. 
However, the extension to infinite-horizon Markov games is highly non-trivial as there is additional ``instability penalty'' in the system that we need to handle; see \pref{sec:analysis} for detailed discussions.

\subsection{Related Work}
\label{subsec: related work}
In this section, we discuss and compare related works on learning two-player zero-sum Markov games.
We refer the readers to a thorough survey by~\citep{zhang2019multi} for other topics in multi-agent reinforcement learning.

\cite{shapley1953stochastic} first introduces the Markov game model and proposes an algorithm analogous to value iteration for solving two-player zero-sum Markov games (with all parameters known).
Later, \cite{hoffman1966nonterminating} propose a policy iteration algorithm, and \cite{pollatschek1969algorithms} propose another policy iteration variant that works better in practice but cannot always converge. With the efforts of \cite{van1978discounted} and \cite{filar1991algorithm}, a slight variant of the \citep{pollatschek1969algorithms} algorithm is proposed in~\citep{filar1991algorithm} and proven to converge. 
In such a full-information setting where all parameters are know, our algorithm has no estimation error and can also be viewed as a new policy-iteration algorithm.

\cite{littman1994markov} initiates the study of competitive reinforcement learning under the framework of Markov games and proposes an extension of the single-player Q-learning algorithm, called {minimax-Q}, which is later proven to converge under some conditions~\citep{szepesvari1999unified}. While minimax-Q can run in a decentralized manner, it is conservative and only converges to the minimax policy but not the best response to the opponent. 

To fix this issue, the work of \cite{bowling2001rational} argues that a desirable multi-agent learning algorithm should have the following two properties simultaneously: \emph{rational} and \emph{convergent}. By their definition, a rational algorithm converges to its
opponent's best response if the opponent converges to a stationary policy,\footnote{It is tempting to consider an even stronger rationality notion, that is, having no regret against an arbitrary opponent. This is, however, known to be computationally hard~\citep{radanovic2019learning, bai2020near}. 
} 
while a convergent algorithm converges to a Nash equilibrium if both agents use it. They propose the WoLF (Win-or-Learn-Fast) algorithm to achieve this goal, albeit only with empirical evidence. Subsequently, \cite{conitzer2007awesome, perolat2018actor, sayin2020fictitious} design decentralized algorithms that provably enjoy these two properties, but only with asymptotic guarantees.


Recently, there is a surge of works that provide finite-time guarantees and characterize the tight sample complexity for finding Nash equilibria~\citep{perolat2015approximate, perolat2016softened, wei2017online, sidford2020solving, xie2020, zhang2020model, bai2020provable, liu2020sharp}. 
These algorithms are all essentially centralized. 
Below, we focus on comparisons with several recent works that propose decentralized algorithms and provide finite-time guarantees. 

\paragraph{Comparison with R-Max \citep{brafman2002r},  UCSG-online \citep{wei2017online} and OMNI-VI-online \citep{xie2020}}
These algorithms, like minimax-Q, converge to the minimax policy instead of the best response to the opponent, even when the opponent is weak (i.e., not using its best policy). In other words, these algorithms are not rational. Another drawback of these algorithms is that the learner has to observe the actions taken by the opponent. Our algorithm, on the other hand, is both rational and agnostic to what the opponent plays.

\paragraph{Comparison with Optimistic Nash V-Learning \citep{bai2020near, tian2020provably}} 
The Optimistic Nash V-Learning algorithm handles the finite-horizon tabular case. It runs an exponential-weight algorithm on each state, with importance-weighted loss/reward estimators. It is unclear whether the dynamics of Optimistic Nash V-Learning leads to last iterate convergence. 
After training, however, Optimistic Nash V-Learning can output a near-optimal non-Markovian policy with size linear in the training time. In contrast, our algorithm exhibits last-iterate convergence, and the output is a simple Markovian policy.
  
\paragraph{Comparison with Smooth-FSP \citep{liu2021policy}}
The Smooth-FSP algorithm handles the function approximation setting. The objective function it optimizes is the original objective plus an entropy regularization term. Because of this additional regularization, the players are only guaranteed to converge to some neighborhood of the minimax policy pair (with a constant radius), even when their gradient estimation error is zero. In contrast, our algorithm converges to the true minimax policy pair when the gradient estimation error goes to zero. 

\paragraph{Comparison with Independent PG \citep{daskalakis2021independent}} \cite{daskalakis2021independent} studies independent policy gradient in the tabular case. To achieve last-iterate convergence, the two players have to use asymmetric learning rates, and only the one with a smaller learning rate converges to the minimax policy. In contrast, the two players of our algorithm are completely symmetric, and they simultaneously converge to the equilibrium set.

\section{Preliminaries}\label{sec:prelim}

We consider a two-player zero-sum discounted Markov game defined by a tuple $(\calS, \calA, \calB, \sig, p, \gamma)$, where: 1) $\calS$ is a finite state space; 2) $\calA$ and $\calB$ are finite action spaces for Player 1 and Player 2 respectively; 3) $\sig$ is the loss (payoff) function for Player 1 (Player 2), with $\sig(s, a, b)\in[0,1]$ specifying how much Player 1 pays to Player 2 if they are at state $s$ and select actions $a$ and $b$ respectively; 4) $p: \calS\times \calA\times \calB\rightarrow \Delta_{\calS}$ is the transition function, with $p(s'|s, a, b)$ being the probability of transitioning to state $s'$ after actions $a$ and $b$ are taken by the two players respectively at state $s$ ($\Delta_{\calS}$ denotes the set of probability distributions over $\calS$);
5) and $\frac{1}{2}\leq \gamma<1$ is a discount factor.\footnote{The discount factor is usually some value close to $1$, so we assume that it is no less than $\frac{1}{2}$ for simplicity. Also note that we consider the discounted setting instead of the finite-horizon episodic setting because the former captures more challenges of this problem (and is also the original setting considered in~\citep{bowling2001rational}). Indeed, in the episodic setting where states have a layered structure, convergence can be directly shown in a layer-by-layer manner; see~\citep{wei2020linear-arxiv}, an early version of~\citep{wei2020linear}. 
}

A stationary policy of Player 1 can be described by a function $\calS\rightarrow \Delta_{\calA}$ that maps each state to an action distribution. 
We use $\x^s\in\Delta_{\calA}$ to denote the action distribution for Player 1 on state $s$, and use $\x=\{\x^s\}_{s\in\calS}$ to denote the complete policy. We define $\y^s$ and $\y=\{\y^s\}_{s\in\calS}$ similarly for Player 2. For notational convenience, we further define $\z^s=(\x^s, \y^s)\in\Delta_{\calA}\times \Delta_{\calB}$ as the concatenated policy of the players on state $s$, and let $\z=\{\z^s\}_{s\in\calS}$.

For a pair of stationary policies $(\x, \y)$ and an initial state $s$, the expected \emph{discounted} \emph{value} that the players pay/gain can be represented as 
\begin{align*}
    V^s_{\x, \y} = \E\left[\sum_{t=1}^\infty \gamma^{t-1} \sig(s_t, a_t, b_t)~\bigg\vert~s_1=s, \quad a_t\sim \x^{s_t}, b_t\sim \y^{s_t}, s_{t+1}\sim p(\cdot|s_t,a_t,b_t), \ \forall t\geq 1\right].
\end{align*}
The \emph{minimax} game value on state $s$ is then defined as 
\begin{align*}
    V^s_\star = \min_{\x}\max_{\y} V^{s}_{\x, \y} = \max_{\y}\min_{\x} V^{s}_{\x, \y}.
\end{align*}
It is known that a pair of stationary policies $(\x_\star, \y_\star)$ attaining the minimax value on state $s$ is necessarily attaining the minimax value on all states~\citep{filar2012competitive},
and we call such $\x_\star$ a \emph{minimax} policy, such $\y_\star$ a \emph{maximin} policy, and such pair a Nash equilibrium.
Further define $\calX^s_{\star}=\{\x^s_{\star} \in \x_{\star} : \text{$\x_{\star}$ is a minimax policy}\}$ and similarly $\calY^s_{\star}=\{\y^s_{\star} \in \y_{\star} : \text{$\y_{\star}$ is a maximin policy} \}$, and denote $\calZ_\star^s = \calX^s_{\star} \times \calY^s_{\star}$.
It is also known that any $\x = \{\x^s\}_{s\in\calS}$ with $\x^s \in \calX^s_{\star}$ for all $s$ is a minimax policy (similarly for $\y$)~\citep{filar2012competitive}.
%

For any $x^s$, we denote its distance from $\calX^s_{\star}$ as $\dist_\star(x^s)= \min_{x_\star^s\in\calX^s_\star}\|x_\star^s-x^s\|$, where $\|v\|$ for a vector $v$ denotes its $L_2$ norm throughout the paper;
similarly, $\dist_\star(y^s)= \min_{y_\star^s\in\calY^s_\star}\|y_\star^s-y^s\|$ and $\dist_\star(z^s)= \min_{z_\star^s\in\calZ_\star^s} \|z_\star^s-z^s\|= \sqrt{\dist_\star^2(x^s) + \dist_\star^2(y^s)}$.\footnote{Note the slight abuse of notation here: the meaning of $\dist_\star(\cdot)$ depends on its input.
} The projection operator for a convex set $\mathcal{U}$ is defined as $\Pi_{\mathcal{U}}\{v\}=\argmin_{u\in\mathcal{U}}\|u-v\|$.


We also define the \emph{Q-function} on state $s$ under policy pair $(\x, \y)$ as
\begin{align*}
    Q^s_{\x,\y}(a,b) = \sig(s,a,b) + \gamma \E_{s'\sim p(\cdot|s,a,b)}\left[V^{s'}_{\x, \y}\right],  
\end{align*}
which can be compactly written as a matrix $\Q^s_{\x, \y}\in\mathbb{R}^{|\calA|\times |\calB|}$ such that $V^s_{x,y}=x^{s^\top}Q^s_{x,y}y^s$.  We write $\Q^s_\star=\Q^{s}_{\x_\star, \y_\star}$ for any minimax/maximin policy pair $(\x_\star, \y_\star)$ (which is unique even if $(\x_\star, \y_\star)$ is not). Finally, $\|Q\|$ for a matrix $Q$ is defined as $\max_{i,j}|Q_{i,j}|$.




\paragraph{Optimistic Gradient Descent Ascent (OGDA)}
As mentioned, our algorithm is based on running an instance of the OGDA algorithm on each state with an appropriate loss/reward function.
To this end, here, following the exposition of~\citep{wei2020linear} we briefly review OGDA for a matrix game defined by a matrix $Q\in\mathbb{R}^{|\calA|\times |\calB|}$.
Specifically, OGDA maintains two sequences of action distributions  $\xp_1, \xp_2, \ldots\in \Delta_\calA$ and $\x_1, \x_2, \ldots \in \Delta_\calA$ for Player 1, and similarly two sequences $\yp_1, \yp_2, \ldots \in \Delta_\calB$ and $\y_1, \y_2, \ldots \in \Delta_\calB$ for Player 2, following the updates below:
\begin{equation}\label{eq:OGDA_matrix_games}
\begin{split}
    &\xp_{t+1} = \Pi_{\Delta_{\calA}}\big\{\xp_{t} - \eta Q\y_t\big\}, \;\;\qquad \qquad 
    \x_{t+1} = \Pi_{\Delta_{\calA}}\big\{\xp_{t+1} - \eta Q\y_t\big\}, \\
    &\yp_{t+1} = \Pi_{\Delta_{\calB}}\big\{\yp_{t} + \eta Q^\top \x_t\big\}, \qquad \qquad 
    \y_{t+1} = \Pi_{\Delta_{\calB}}\big\{\yp_{t+1} + \eta Q^\top \x_t\big\},
\end{split}
\end{equation}
where $\eta$ is some learning rate.
As one can see, unlike the standard Gradient Descent Ascent algorithm which simply sets $(\x_t, \y_t)=(\xp_t, \yp_t)$, OGDA takes a further descent/ascent step using the latest gradient to obtain $(x_t, y_t)$, which is then used to evaluate the gradient (of the function $f(\x,\y) = \x^\top Q \y$).
\citet{wei2020linear} prove that the iterate $(\xp_t, \yp_t)$ (or $(\x_t, \y_t)$) converges to the set of Nash equilibria of the matrix game at a linear rate, which motivates us to generalize it to Markov games.
As we show in the following sections, however, the extensions of both the algorithm and the analysis are highly non-trivial.

We remark that while \citet{wei2020linear} also analyze the last-iterate convergence of another algorithm called Optimistic Multiplicative Weight Update (OMWU), which is even more commonly used in finite-action games, they also show that the theoretical guarantees of OMWU hold under more limited assumptions (e.g., requiring the uniqueness of the equilibrium), and its empirical performance is also inferior to that of OGDA.
We therefore only extend the latter to Markov games.

\section{Algorithm and Main Results}
\label{sec: algorithm section}

A natural idea to extend OGDA to Markov games is to run the same algorithm described in \pref{sec:prelim} for each state $s$ with the game matrix $Q$ being $Q_{\x_t, \y_t}^s$.
However, an important difference is that now the game matrix is {\it changing} over time.
Indeed, if the polices are changing rapidly for subsequent states, the game matrix $Q_{\x_t, \y_t}^s$ will also be changing rapidly, which makes the update on state $s$ highly unstable and in turn causes similar issues for previous states.

To resolve this issue, we propose to have a {\it critic} slowly learn the value function for each state. 
Specifically, for each state $s$, the critic maintains a sequence of values $V_0^s = 0, V_1^s, V_2^s, \ldots$.
During iteration $t$, instead of using $Q_{\x_t, \y_t}^s$ as the game matrix for state $s$, we use $Q_t^s$ defined via $Q_t^s(a,b) = \sig(s,a,b) + \gamma \E_{s'\sim p(\cdot|s,a,b)} [V_{t-1}^{s'} ]$.
Ideally, OGDA would then take the role of an {\it actor} and compute $\x_{t+1}^s$ and $\xp_{t+1}^s$ using the gradient $Q_t^s y_t^s$ (and similarly $\y_{t+1}^s$ and $\yp_{t+1}^s$ using the gradient $Q_t^{s^\top} x_t^s$).
Since such exact gradient information is often unknown, we only require the algorithm to come up with estimations $\ell_t^s$ and $r_t^s$ such that $\|\ell^s_t - Q^s_t y^s_t\| \leq \err$ and $\|r^s_t  - Q_t^{s^\top} x^s_t\| \leq \err$ for some prespecified error $\err$ (more discussions in \pref{sec:est}). See updates \pref{eq: update 1}-\pref{eq: update 4} in \pref{algo: main alg}.
Note that similar to~\citep{wei2020linear}, we adopt a constant learning rate $\eta$ (independent of the number of iterations) in these updates.

At the end of each iteration $t$, the critic then updates the value function via $V^s_{t} = (1-\alpha_t)V^s_{t-1} + \alpha_t \pay_t^s$, where $\pay_t^s$ is an estimation of $x_t^{s^\top} Q^s_t y^s_t$ such that $|\pay^s_t - x_t^{s^\top} Q^s_t y^s_t| \leq \err$.\footnote{%
For simplicity, here we assume that the two players share the same estimator $\pay_t^s$ (and thus same $V^s_t$ and $Q^s_t$). However, our analysis works even if they maintain different versions of $\pay_t^s$, as long as they are $\err$-close to $x_t^{s^\top} Q^s_t y^s_t$ with respect to their own $Q^s_t$.}
To stabilize the game matrix, we require the learning rate $\alpha_t$ to decrease in $t$ and go to zero.
Most of our analysis is conducted under this general condition, and the final convergence rate depends on the concrete form of $\alpha_t$, which we set to $\alpha_t=\frac{H+1}{H+t}$ with $H=\frac{2}{1-\gamma}$ inspired by~\citep{jin2018q} (there could be a different choice leading to a better convergence though).

\begin{algorithm}[t]
    \caption{Optimistic Gradient Descent/Ascent for Markov Games}
    \label{algo: main alg}
    \textbf{Parameters}: $\gamma\in[\frac{1}{2}, 1), \eta \leq \frac{1}{10^4}\sqrt{\frac{(1-\gamma)^5}{S}}$, $\err\in \left[0, \frac{1}{1-\gamma}\right]$. \\
    \textbf{Parameters}: a non-increasing sequence $\{\alpha_t\}_{t=1}^T$ that goes to zero.  \\ 
    \textbf{Initialization}: 
    $\forall s \in \calS$, arbitrarily initialize $\xp_1^s=\x_1^s \in \Delta_{\calA} $ and $\yp_1^s=\y_1^s\in\Delta_\calB$,
    and set $V_0^s\leftarrow 0$. \\
    \For{$t=1,\ldots, T$}{
        For all $s$, 
        define 
        $\Q_t^s\in\mathbb{R}^{|\calA|\times|\calB|}$ as
        \begin{align*}
        Q_t^s(a,b) &\triangleq  
           \sig(s,a,b) + \gamma \E_{s'\sim p(\cdot|s,a,b)} \left[V_{t-1}^{s'} \right],
        \end{align*} 
        and update  
        \begin{align} 
            \xp_{t+1}^s &= \Pi_{\Delta_\calA}\Big\{\xp_t^s - \eta \ell_t^s \Big\},    \label{eq: update 1}\\
            \x_{t+1}^s &= \Pi_{\Delta_\calA}\Big\{\xp_{t+1}^s - \eta \ell_t^s \Big\},  \label{eq: update 2}\\
            \yp_{t+1}^s &= \Pi_{\Delta_\calB}\Big\{\yp_t^s + \eta r_t^s \Big\},  \label{eq: update 3} \\ 
            \y_{t+1}^s &= \Pi_{\Delta_\calB}\Big\{\yp_{t+1}^s + \eta r_t^s \Big\},  \label{eq: update 4}\\
            V^s_{t} &= (1-\alpha_t)V^s_{t-1} + \alpha_t \pay_t^s,  \label{eq: update 5}
        \end{align}
        where $\ell^s_t, r^s_t$, and $\pay^s_t$ are $\err$-approximations of $Q^s_t y^s_t$, $Q_t^{s^\top} x^s_t$, and $x_t^{s^\top} Q^s_t y^s_t$ respectively, such that $\|\ell^s_t - Q^s_t y^s_t\| \leq \err$, $\|r^s_t  - Q_t^{s^\top} x^s_t\| \leq \err$, and $|\pay^s_t - x_t^{s^\top} Q^s_t y^s_t| \leq \err$.
    }
\end{algorithm}


Our main results are the following two theorems on the last-iterate convergence of \pref{algo: main alg}.
\begin{theorem}[Average duality-gap convergence]\label{thm: main 1}
    \pref{algo: main alg} with the choice of $\alpha_t=\frac{H+1}{H+t}$ where $H=\frac{2}{1-\gamma}$ guarantees 
    \begin{align*}
        \frac{1}{T}\sum_{t=1}^T \max_{s,x',y'} \left(V^s_{\xp_t, y'} -V^s_{x', \yp_t} \right) = \order\left(\frac{|\calS|}{\eta(1-\gamma)^2}\sqrt{\frac{\log T}{T}} + \frac{|\calS|\sqrt{\err}}{\sqrt{\eta}(1-\gamma)^2}\right).   
    \end{align*}
\end{theorem}

\begin{theorem}[Last-iterate convergence]\label{thm: main 2}
    \pref{algo: main alg} with the choice of $\alpha_t=\frac{H+1}{H+t}$ where $H=\frac{2}{1-\gamma}$ guarantees with $\zp_T^s = (\xp_T^s, \yp_T^s)$,
    \begin{align*}
        \frac{1}{|\calS|} \sum_{s\in\calS} \dist_\star^2(\zp_T^s) = \order\left( \frac{|\calS|^2}{\eta^4C^4(1-\gamma)^4 T} + \frac{\err }{\eta C^2(1-\gamma)^3} \right), 
    \end{align*}
    where $C>0$ is a problem-dependent constant (that always exists) satisfying: for all state $s$ and all policy pair $z=(x,y)$, 
$
        \max_{x', y'} \left(x^{s^\top} Q_\star^s y'^s - x'^{s^\top} Q_\star^s y^s\right) \geq C\dist_\star(z^s)
$.
\end{theorem}

\pref{thm: main 1} shows that the average duality-gap for each state $s$ goes to zero when both $1/T$ and $\err$ go to zero, though it does not show the convergence of the policy.
\pref{thm: main 2}, on the other hand, shows a concrete finite-time convergence rate on the distance of $\zp_T^s$ from the equilibrium set, which goes down at the rate of $1/T$ up to the estimation error $\err$.
The problem-dependent constant $C$ is similar to the matrix game case analyzed in~\citep{wei2020linear}, as we will discuss in \pref{sec:analysis}.
As far as we know, this is the first symmetric algorithm with finite-time last-iterate convergence for both players simultaneously.

\subsection{Estimation}\label{sec:est}
In the full-information setting where all parameters of the Markov game are given, we can calculate the exact value of $Q^s_t y^s_t$, $Q_t^{s^\top} x^s_t$, and $x_t^{s^\top} Q^s_t y^s_t$, making $\err = 0$.
In this case,  our algorithm is essentially a new policy-iteration style algorithm for solving Markov games.
However, in a learning setting where the parameters are unknown, the players need to estimate these quantities based on any feedback from the environments.
Here, we discuss how to do so when the players only observe their current state and their loss/reward after taking an action.

Specifically, in iteration $t$ of our algorithm and with $(\x_t, \y_t)$ at hand,
the two players interact with each other for a sequence of $\LL$ steps, following a mixed strategy with a certain amount of uniform exploration defined via: $\xtil^s_t(a) = \left(1-\frac{\err'}{2}\right)x^s_t(a) + \frac{\err'}{2|\calA|}$ and $\ytil^s_t(b) = \left(1-\frac{\err'}{2}\right)y^s_t(b) + \frac{\err'}{2|\calB|}$, where $\err'=(1-\gamma)\err$.
This generates a sequence of observations $\{(s_i, a_i, \sig(s_i, a_i,b_i))\}_{i=1}^{\LL}$ for Player 1 and similarly a sequence of observations $\{(s_i, b_i, \sig(s_i, a_i,b_i))\}_{i=1}^{\LL}$ for Player 2, where $a_i \sim \xtil^{s_i}_t$, $b_i \sim \ytil^{s_i}_t$, and $s_{i+1} \sim p(\cdot|s_i, a_i, b_i)$.
Then we construct the estimators as follows:
\begin{align}
    \ell_t^s(a) &= \frac{\sum_{i=1}^\LL \one[s_i=s, a_i=a]\left(\sig(s, a, b_i) + \gamma V_{t-1}^{s_{i+1}}\right)}{\sum_{i=1}^\LL \one[s_i=s, a_i=a] },\label{eq: est1} \\
    r_t^s(b) &= \frac{\sum_{i=1}^\LL\one[s_i=s, b_i=b] \left(\sig(s, a_i, b) + \gamma V_{t-1}^{s_{i+1}}\right)}{\sum_{i=1}^\LL \one[s_i=s, b_i=b] },\label{eq: est2}\\
    \pay^s_t &= \frac{\sum_{i=1}^\LL \one[s_i=s]\left(\sig(s, a_i, b_i)+\gamma V_{t-1}^{s_{i+1}}\right)}{\sum_{i=1}^\LL \one[s_i=s]}. \label{eq: est3}
\end{align}
(If any of the denominator is zero, define the corresponding estimator as zero.)
To make sure that these are accurate estimators for every state, we naturally need to ensure that every state is visited often enough.
To this end, we make the following assumption similar to~\citep{ortner2007logarithmic}, which essentially requires that the induced Markov chain under any stationary policy pair is irreducible.

\begin{assumption}\label{assum:irreducibility}
There exists $\mu > 0$ such that $\frac{1}{\mu}=\max_{x, y} \max_{s, s'}  T^{s\rightarrow s'}_{x,y}$, where $T^{s\rightarrow s'}_{x,y}$ is the expected time to reach $s'$ from $s$ following the policy pair $(x,y)$.
\end{assumption}

Under this assumption, the following theorem shows that taking $\LL \approx 1/\err^3$ is enough to ensure the accuracy of the estimators (see \pref{app: samples} for the proof).

\begin{theorem}\label{thm: samples}
Suppose that \pref{assum:irreducibility} holds and $\LL=\widetilde{\Omega}\left(\frac{|\calA|^3+|\calB|^3}{(1-\gamma)\mu\err^3}\log^2(T/\delta)\right)$.\footnote{We use $\widetilde{\Omega}$ to hide logarithmic factors except for $\log(T)$ and $\log(1/\delta)$.} 
Then the estimators \pref{eq: est1}, \pref{eq: est2}, and \pref{eq: est3} ensure that with probability at least $1-\delta$, $\|\ell^s_t - Q^s_t y^s_t\|$, $\|r^s_t  - Q_t^{s^\top} x^s_t\|$, and $|\pay^s_t - x_t^{s^\top} Q^s_t y^s_t|$ are all of order $\order(\err)$ for all $t$.
\end{theorem}

Together with \pref{thm: main 1} and \pref{thm: main 2}, given a fixed number of interactions between the players, we can now determine optimally how many iterations we should run our algorithm (and consequently how large we should set $\err$).
Equivalently, we show below how many iterations or total interactions are need to achieve a certain accuracy.
(The choice of $\alpha_t$ is the same as in \pref{thm: main 1} and \pref{thm: main 2}.)

\begin{corollary}\label{col: thm2}
If \pref{assum:irreducibility} holds, then running \pref{algo: main alg} with estimators \pref{eq: est1}, \pref{eq: est2}, \pref{eq: est3} and $\LL=\widetilde{\Omega}\left(\tfrac{(|\calA|^3+|\calB|^3)|\calS|^6}{(1-\gamma)^{13}\mu \eta^3\arx^6}\log^2(T/\delta)\right)$ for $T=\widetilde{\Omega}\left(\frac{|\calS|^2}{\eta^2(1-\gamma)^4 \arx^2}\right)$ iterations ensures with probability at least $1-\delta$, $\frac{1}{T}\sum_{t=1}^T \max_{s,x',y'} (V^s_{\xp_t, y'} -V^s_{x', \yp_t} )\le \arx$. Ignoring other dependence, this requires $\widetilde{\Omega}({1}/{\arx^8})$ interactions in total.
\end{corollary}

\begin{corollary}\label{col: thm1}
    If \pref{assum:irreducibility} holds, then running \pref{algo: main alg} with estimators \pref{eq: est1}, \pref{eq: est2}, \pref{eq: est3} and $\LL=\widetilde{\Omega}\left(\frac{|\calA|^3+|\calB|^3}{(1-\gamma)^{10}\mu \eta^3C^6\arx^3}\log^2(T/\delta)\right)$ for $T=\Omega\left(\frac{|\calS|^2}{\eta^4C^4(1-\gamma)^4 \arx}\right)$ iterations ensures with probability at least $1-\delta$, $\frac{1}{|\calS|} \sum_{s\in\calS} \dist_\star^2(\zp_T^s)\le \arx$. Ignoring other dependence, this requires $\widetilde{\Omega}({1}/{\arx^4})$ interactions in total.
\end{corollary}

\subsection{Rationality}
Finally, we argue that from the perspective of a single player (take Player 1 as an example), our algorithm is also rational, in the sense that it allows Player 1 to converge to the best response to her opponent if Player 2 is not applying our algorithm but instead uses an arbitrary stationary policy.\footnote{The rationality defined by \cite{bowling2001rational} requires that the learner converges to the best response as long as the opponent \emph{converges} to a stationary policy. While our algorithm does handle this case, as a proof of concept, we only consider the simpler scenario where the opponent simply uses a stationary policy. }
We show this single-player-perspective version in \pref{algo: main alg single}, where Player 1 still follows the updates \pref{eq: update 1}, \pref{eq: update 2}, and \pref{eq: update 5}, while $y_t$ is fixed to a stationary policy $y$ used by Player 2.

In fact, thanks to the {\it agnostic} nature of our algorithm, rationality is essentially an implication of the convergence property.
To see this, consider a modified two-player Markov game with the difference being that the opponent has only a single action (call it $1$) on each state, 
the loss function is redefined as $\gametwo{\sig}(s,a,1)=\E_{b\sim y^s}[\sig(s,a,b)]$, and the transition kernel is redefined as $\gametwo{p}(s'|s,a,1)=\E_{b\sim y^s}[p(s'|s,a,b)]$. 
It is straightforward to see that following our algorithm, Player 1's behaviors in the original game and in the modified game are exactly the same.
On the other hand, in the modified game, since Player 2 has only one action (and thus one strategy), she can also be seen as using our algorithm.
Therefore, we can apply our convergent guarantees to the modified game, and since the minimax policy in the modified game is exactly the best response in the original game, we know that Player 1 indeed converges to the best response.
We summarize these rationality guarantees in the following theorem, with the formal proof deferred to \pref{app: rationality}.

\begin{theorem}\label{thm: rational-1}
    \pref{algo: main alg single} with the choice of $\alpha_t=\frac{H+1}{H+t}$ where $H=\frac{2}{1-\gamma}$ guarantees
    \begin{align*}
        \frac{1}{T}\sum_{t=1}^T \max_{s,x'} \left(V^s_{\xp_t, y^s} -V^s_{x', y^s} \right) = \order\left(\frac{|\calS|}{\eta(1-\gamma)^2}\sqrt{\frac{\log T}{T}} + \frac{|\calS|\sqrt{\err}}{\sqrt{\eta}(1-\gamma)^2}\right),
    \end{align*}
   and for $\calX_{BR} = \left\{x: V^s_{x,y} = \min_{x'} V^s_{x',y}, \forall s\in\calS \right\}$ and some problem-dependent constant $C' > 0$,
    \begin{align*}
        \frac{1}{|\calS|} \sum_{s\in\calS} \|\xp_T^s -  \Pi_{\calX_{BR}}\{\xp_T^s\}\|^2 = \order\left( \frac{|\calS|^2}{\eta^4C'^4(1-\gamma)^4 T} + \frac{\err }{\eta C'^2(1-\gamma)^3} \right).
    \end{align*}
\end{theorem}

\section{Analysis Overview}\label{sec:analysis}
In this section, we give an overview of how we analyze \pref{algo: main alg} and prove \pref{thm: main 1} and \pref{thm: main 2}.
We start by giving a quick review of the analysis of \citep{wei2020linear} for matrix games, and then highlight how we overcome the challenges when generalizing it to Markov games.

\paragraph{Review for matrix games}
Recall the update in \pref{eq:OGDA_matrix_games} for a fixed matrix $Q$. \citet{wei2020linear} show the following two convergence guarantees:
\begin{enumerate}
    \item Average duality-gap convergence: 
    \begin{align}
        \frac{1}{T}\sum_{t=1}^T \Delta(\zp_t) = \order\left(\frac{1}{\eta \sqrt{T}}\right)   \label{eq: matrix game duality gap}
    \end{align}
    where $\Delta(z)=\max_{x', y'}\left( x^\top Qy' - x'^{\top}Qy \right)$ is the duality gap of $z=(x,y)$. 
    \item Last-iterate convergence: 
    \begin{align}
        \dist_{\star}^2(\zp_t) \leq C_1\dist_{\star}^2(\zp_1)\left(1+\eta^2 C^2\right)^{-t}    \label{eq: matrix game last itera}
    \end{align}
    where $\dist_{\star}(z)$ is the distance from $z$ to the set of equilibria, $C_1$ is a universal constant, and $C>0$ is a positive constant that depends on $Q$.\footnote{%
    This is not to be confused with the constant $C$ in \pref{thm: main 2}. We overload the notation because they indeed play the same role in the analysis.
}
\end{enumerate}
The analysis of \citep{wei2020linear} starts from the following single-step inequality that follows the standard Online Mirror Descent analysis and describes the relation between $\dist_\star^2(\zp_{t+1})$ and $\dist_\star^2(\zp_{t})$: 
\begin{align}
     \dist_\star^2(\zp_{t+1}) 
     &\leq \dist_\star^2(\zp_{t}) + \underbrace{\eta^2 \norm{z_t-z_{t-1}}^2}_{\text{instability penalty}} - \underbrace{\left(\norm{\zp_{t+1}-z_t}^2 + \norm{z_t-\zp_{t}}^2\right)}_{\text{instability bonus}}. \label{eq: single step for matrix game} 
\end{align}
The instability penalty term makes $\dist_\star^2(\zp_{t+1})$ larger if $\norm{z_t - z_{t-1}}$ is large, while the instability bonus term makes $\dist_\star^2(\zp_{t+1})$ smaller if either $\norm{\zp_{t+1}-z_t}$ or $\norm{z_t-\zp_t}$ is large. 
To obtain \pref{eq: matrix game duality gap}, \cite{wei2020linear} make the observation that the instability bonus term is lower bounded by a constant times the squared duality gap of $\zp_{t+1}$, that is, $\norm{\zp_{t+1}-z_t}^2 + \norm{z_t-\zp_{t}}^2 \gtrsim \eta^2 \Delta^2(\zp_{t+1})$, and thus
\begin{align}
     \dist_\star^2(\zp_{t+1}) 
     &\leq \dist_\star^2(\zp_{t}) + \underbrace{\eta^2 \norm{z_t-z_{t-1}}^2}_{\text{instability penalty}} - \underbrace{\frac{1}{2}\left(\norm{\zp_{t+1}-z_t}^2 + \norm{z_t-\zp_{t}}^2\right)}_{\frac{1}{2}\ \text{instability bonus}} - \Omega(\eta^2\Delta^2(\zp_{t+1})). \label{eq: single step for matrix game 2} 
\end{align}
By taking $\eta\leq \frac{1}{8}$, summing over $t$, canceling the penalty term with the bonus term, telescoping and rearranging, we get $\sum_{t=1}^T \Delta^2(\zp_t) \leq \order(1/\eta^2)$.
An application of Cauchy-Schwarz inequality then proves \pref{eq: matrix game duality gap}. 

To further obtain \pref{eq: matrix game last itera}, \cite{wei2020linear} prove that there exists some problem-dependent constant $C>0$ such that for all $z$, $\Delta(z)\geq C\dist_\star(z)$. This, when combined with \pref{eq: single step for matrix game 2}, shows 
\begin{align}
    \dist_\star^2(\zp_{t+1})\leq \frac{\dist_\star^2(\zp_t)}{1+\Omega(\eta^2C^2)} + \eta^2 \|z_t-z_{t-1}\|^2 - \Omega\left(\norm{\zp_{t+1}-z_t}^2 + \norm{z_t-\zp_{t}}^2\right).   \label{eq: matrix game last-iterate}
\end{align}
By upper bounding $\|z_t-z_{t-1}\|^2\leq 2\|z_t-\zp_t\|^2 + 2\|\zp_t - z_{t-1}\|^2$ and rearranging, they further obtain: 
\begin{align}
    \dist_\star^2(\zp_{t+1}) + c'\norm{\zp_{t+1}-z_t}^2 + c'\norm{z_t-\zp_{t}}^2\leq \frac{\dist_\star^2(\zp_{t}) + c'\norm{\zp_{t}-z_{t-1}}^2 + c'\norm{z_{t-1}-\zp_{t-1}}^2}{1+\Omega(\eta^2C^2)}   \label{eq: linear lastiterate}
\end{align}
for some universal constant $c'$, which clearly indicates the linear convergence of $\dist_\star^2(\zp_t)$ and hence proves \pref{eq: matrix game last itera}. 

\paragraph{Overview of our proofs}
We are now ready to show the high-level ideas of our analysis.
For simplicity, we consider the case with $\err = 0$ and also assume that there is a unique equilibrium $(\x_\star, \y_\star)$ (these assumptions are removed in the formal proofs).
Our analysis follows the steps below.

\paragraph{Step 1 (\pref{app: step-1})}
Similar to \pref{eq: single step for matrix game}, we conduct a single-step analysis for OGDA in Markov games (\pref{lemma: distance decrease}), which shows for all state $s$: 
\begin{align}
    \dist_\star^2(\zp_{t+1}^s)
    &\leq \dist_\star^2(\zp_{t}^s) + \eta^2 \norm{z_t^s-z_{t-1}^s}^2 - \left(\norm{\zp_{t+1}^s-z_t^s}^2 + \norm{z_t^s - \zp_{t}^s}^2\right) \nonumber \\
    & \qquad \qquad \qquad + 8\eta^2\norm{Q^s_t-Q^s_{t+1}}^2 + 4\eta\norm{Q_t^s - Q^s_\star}.   \label{eq: single step MG}
\end{align}
Comparing this with \pref{eq: single step for matrix game}, we see that, importantly, since the game matrix $Q_t^s$ is changing over time, we have two extra \emph{instability penalty} terms: $\eta^2\norm{Q_t^s-Q_{t-1}^s}^2$ and $\eta\norm{Q_t^s-Q_\star^s}$. 
Our hope is to further upper bound these two penalty terms by something related to $\|z_t^s-z_{t+1}^s\|^2$, so that they can again be canceled by the bonus term $- (\norm{\zp_{t+1}^s-z_t^s}^2 + \norm{z_t^s - \zp_{t}^s}^2)$. 
Indeed, in Steps 3-5, we show that part of them can be bounded by a weighted sum of $\{\|z_\tau^{s'}-z_{\tau+1}^{s'}\|^2\}_{s'\in\calS, \tau\leq t}$.


\paragraph{Step 2 (\pref{app: step-2}): Lower bounding $\norm{\zp_{t+1}^s-z_t^s}^2 + \norm{z_t^s - \zp_{t}^s}^2$. } 
As in \pref{eq: single step for matrix game 2}, we aim to lower bound the instability bonus term by the duality gap. However, since the updates are based on $Q^s_t$ instead of $Q^s_\star$, we can only relate the bonus term to the duality gap with respect to $Q^s_t$. To further relate this to the duality gap with respect to $Q^s_\star$, we pay a quantity related to $\|Q^s_t-Q^s_\star\|$. Formally, we show in \pref{lem: step-2}:
\begin{align*}
    \norm{\zp_{t+1}^s-z_t^s}^2 + \norm{z_t^s - \zp_{t}^s}^2 \gtrsim \Omega(\eta^2\Delta^2(\zp_{t+1}^s)) - \order(\eta\|Q_t^s - Q_\star^s\|), 
\end{align*}
where $\Delta(z^s)\triangleq \max_{x'^s, y'^s} \left( x^{s^\top} Q_\star^s y'^s - x'^{s^\top} Q_\star^s y^s \right)$ is the duality gap on state $s$ with respect to $Q^s_\star$.\footnote{Similar to the notation $\dist_\star(\cdot)$, we also omit writing the $s$ dependence for the function $\Delta(\cdot)$.} 

\paragraph{Step 3 (\pref{app: step-3}): Upper bounding $\norm{Q_{t+1}^s-Q_{t}^s}^2$. } $\norm{Q_{t+1}^s - Q^s_{t}}^2$ is upper bounded by $\gamma^2 \max_{s'}(V_{t}^{s'}-V_{t-1}^{s'})^2$ by the definition of $Q_t^s$. Furthermore, $V_{t}^{s'}-V_{t-1}^{s'}$ is a weighted sum of $\{\pay_{\tau}^{s'} - \pay_{\tau-1}^{s'}\}_{\tau=1}^{t-1}$ by the definition of $V_{t}^{s'}$, and also $\pay_{\tau}^{s'} - \pay_{\tau-1}^{s'}= x_\tau^{s'^\top}Q_\tau^{s'}y_\tau^{s'} - x_{\tau-1}^{s'^\top}Q_{\tau-1}^{s'}y_{\tau-1}^{s'}=\order(\|z_\tau^{s'}-z_{\tau-1}^{s'}\| + \|Q^{s'}_{\tau} - Q^{s'}_{\tau-1}\|)$. In sum, one can upper bound $\|Q^s_{t+1}-Q^s_{t}\|^2$ by a weighted sum of $ \|z_{\tau}^{s'}-z_{\tau-1}^{s'}\|^2$ and 
$ \|Q_{\tau}^{s'}-Q_{\tau-1}^{s'}\|^2$. After formalizing the above relations, we obtain the following inequality (see \pref{lem: key lemma 1}): 
\begin{align}
    \norm{Q^s_{t+1}-Q^s_t}^2 \leq \max_{s'} \frac{8\gamma^2}{(1-\gamma)^3} \sum_{\tau=1}^t \alpha^\tau_t \|z_\tau^{s'}-z_{\tau-1}^{s'}\|^2 + \max_{s'}\frac{2\gamma^2}{1+\gamma}\sum_{\tau=1}^t\alpha^\tau_t \|Q_\tau^{s'}-Q_{\tau-1}^{s'}\|^2   \label{eq: decompose qt+1 qt}
\end{align}
for some coefficient $\alpha^\tau_t$ defined in \pref{app: aux coeff}. With recursive expansion, the above implies that $\norm{Q^s_{t+1}-Q^s_t}^2$ can be upper bounded by a weighted sum of $\|z_\tau^{s'} - z_{\tau-1}^{s'}\|^2$ for $s'\in\calS$ and $\tau\leq t$. 

\paragraph{Step 4 (\pref{app: step-4-5}): Upper bounding $\|Q_t^s - Q_\star^s\|$ (Part 1). } 
We first upper bound $\|Q_t^s - Q_\star^s\|$ with respect to the following weighted-regret quantity
\begin{align*}
    \barReg_t\triangleq  \max_{s}\max\left\{\sum_{\tau=1}^t  \alpha^\tau_t(x_\tau^s - x_\star^s)^\top Q_\tau^s y_\tau^s,\ \ \sum_{\tau=1}^t  \alpha^\tau_t x_\tau^{s^\top} Q_\tau^s (y_\star^s - y_\tau^s) \right\}.
\end{align*}
To do so, we define $\gap_t=\max_s \|Q_t^s-Q_\star^s\|$ and show for the same coefficient $\alpha^\tau_t$ mentioned earlier,
\begin{align*}
    V^{s}_{t}
    &= \sum_{\tau=1}^{t}\alpha^\tau_{t} \pay_\tau^s = \sum_{\tau=1}^{t}\alpha^\tau_{t} x_{\tau}^{s^\top} Q_{\tau}^{s} y_{\tau}^{s} 
    \leq \sum_{\tau=1}^{t} \alpha^\tau_{t} x_{\star}^{s^\top} Q_{\tau}^{s}y_\tau^{s} + \barReg_{t}
    \leq \sum_{\tau=1}^{t} \alpha^\tau_{t}x_{\star}^{s^\top} Q_{\star}^{s}y_\tau^{s} + \sum_{\tau=1}^{t}  \alpha^{\tau}_{t}\gap_\tau + \barReg_{t} \\
    &\leq \sum_{\tau=1}^{t} \alpha^\tau_{t}x_{\star}^{s^\top} Q_{\star}^{s}y_\star^{s} + \sum_{\tau=1}^{t}  \alpha^{\tau}_{t}\gap_\tau + \barReg_{t} 
    = V^{s}_\star + \sum_{\tau=1}^{t}  \alpha^{\tau}_{t}\gap_\tau + \barReg_{t}  
\end{align*}
where the last inequality is by the fact $\sum_{\tau=1}^t\alpha^\tau_t=1$. Using the definition of $Q^s_t$ again, we then have $Q^s_{t+1}(a,b) - Q^s_\star(a,b)=\gamma \E_{s'\sim p(\cdot|s,a,b)}\left[ V_{t}^{s'} - V_\star^{s'} \right]\leq \gamma(\sum_{\tau=1}^{t}\alpha^\tau_{t}\gap_\tau + \barReg_t)$. By the same reasoning, we can also show $Q^s_{t+1}(a,b) - Q^s_\star(a,b) \geq -\gamma(\sum_{\tau=1}^{t}\alpha^\tau_{t}\gap_\tau + \barReg_t)$, and therefore we obtain the following recursive relation (\pref{lem: optimistic lemma discount})
\begin{align}
    \gap_{t+1} = \max_{s} \|Q^s_{t+1}-Q^s_\star\|\leq \gamma\left(\sum_{\tau=1}^{t} \alpha^{\tau}_{t}\gap_\tau + \barReg_{t} \right).  \label{eq: gap and Reg} 
\end{align}

\paragraph{Step 5 (\pref{app: step-4-5}): Upper bounding $\|Q_t^s - Q_\star^s\|$ (Part 2). }
In this step, we further relate $\barReg_t$ to $\{\|z_{\tau}^{s'}-z_{\tau-1}^{s'}\|^2\}_{\tau\leq t, s'\in\calS}$.
From a one-step regret analysis of OGDA, we have the following (for Player 1): 
\begin{align*}
        \left(x_t^s - x_{\star}^s\right)^\top Q^s_t y^s_{t} &\leq \frac{1}{2\eta}\Big(\dist_{\star}^2(\xp^s_t) - \dist_{\star}^2(\xp_{t+1}^s)\Big) + \frac{4\eta}{(1-\gamma)^2}  \|y^s_{t} - y^s_{t-1}\|^2 + 4\eta \|Q^s_{t} - Q^s_{t-1}\|^2. 
\end{align*}
Recall that $\barReg_t$ is defined via a weighted sum of the left-hand side above with weights $\alpha^\tau_t$. Therefore, we take the weighted sum of the above and bound $\sum_{\tau=1}^t \alpha^\tau_t (x_\tau^s - x_\star^s)^\top Q_\tau^s y_\tau^s $ by
\begin{align}
    &\;\; \frac{\alpha^1_t\dist_\star^2(\xp_1^s)}{2\eta} + \sum_{\tau=1}^t \frac{\alpha^\tau_t}{2\eta} \left(\dist_{\star}^2(\xp_\tau^s) - \dist_{\star}^2(\xp_{\tau+1}^s)\right)  \nonumber \\
     &\qquad\qquad    + \frac{4\eta}{(1-\gamma)^2} \sum_{\tau=1}^t  \alpha^\tau_t \|y_\tau^s-y_{\tau-1}^s\|^2 
           +  4\eta \sum_{\tau=1}^t \alpha^\tau_t \|Q_\tau^{s}-Q_{\tau-1}^s\|^2 \nonumber  \\
    &\leq \underbrace{\frac{1}{2\eta}\sum_{\tau=1}^t \alpha^\tau_t \alpha_{\tau-1} \dist_\star^2(\zp_\tau^s)}_{\term_1}  + \underbrace{\frac{4\eta}{(1-\gamma)^2} \sum_{\tau=1}^t  \alpha^\tau_t \|z_\tau^s-z_{\tau-1}^s\|^2}_{\term_2} 
           +  \underbrace{4\eta \sum_{\tau=1}^t \alpha^\tau_t \|Q_\tau^{s}-Q_{\tau-1}^s\|^2}_{\term_3}
\label{eq: upper bounde barReg}
\end{align}
where in the inequality we rearrange the first summation and use the fact $\alpha^\tau_{t} - \alpha^{\tau-1}_{t}\leq \alpha_{\tau-1}\alpha^\tau_t$ (see the formal proof in \pref{lem: interval regret discount nonuniform}). Since the case for $\sum_{\tau=1}^t \alpha^\tau_t x_\tau^{s^\top} Q_\tau^s (y_\star^s - y_\tau^s)$ is similar, by the definition of $\barReg_t$, we conclude that $\barReg_t$ is upper bounded by the maximum over $s$ of the sum of the three terms in \pref{eq: upper bounde barReg}. Note that, $\term_2$ is itself a weighted sum of $\{\|z_\tau^s - z_{\tau-1}^s\|^2\}_{\tau\leq t}$, and $\term_3$ can also be upper bounded by a weighted sum of $\{\|z_\tau^{s'} - z_{\tau-1}^{s'}\|^2\}_{\tau\leq t, s'\in\calS}$ as we already showed in Step 3. 

\paragraph{Combining all steps.} Summing up \pref{eq: single step MG} over all $s$, and based on all earlier discussions, we have
\begin{align}
    &\sum_{s}\dist_\star^2(\zp_{t+1}^s)
    \leq \sum_s \dist_{\star}^2(\zp_t^s) + \underbrace{\sum_{\tau=1}^t \sum_{s}\mu_\tau^{s}\alpha_{\tau-1}\dist_{\star}^2(\zp_\tau^{s})}_{\term_4} + \underbrace{\sum_{\tau=1}^t \sum_{s}\nu_\tau^{s}\|z_\tau^{s}-z_{\tau-1}^{s}\|^2}_{\term_5} \nonumber \\
    &\qquad \qquad \qquad - \underbrace{\frac{1}{2}\sum_s\left(\norm{\zp_{t+1}^s-z_t^s}^2 + \norm{z_t^s - \zp_{t}^s}^2\right)}_{\term_6} - \Omega\left(\eta^2\sum_{s} \Delta^2(\zp_{t+1}^s) \right)   \label{eq: aggre}
\end{align}
for some weights $\mu^{s}_\tau$ and $\nu^{s}_\tau$ (a large part of the analysis is devoted to precisely calculating these weights). Here, the $-\Omega\left(\eta^2\sum_{s} \Delta^2(\zp_{t+1}^s) \right)$ term comes from Step 2; $\term_4$ is a weighted sum of $\{\alpha_{\tau-1}\dist_{\star}^2(\zp_\tau^{s'})\}_{\tau\leq t, s'\in\calS}$ that comes from $\term_1$ in Step 5; $\term_5$ is a weighed sum of $\{\|z_\tau^{s'}-z_{
\tau-1}^{s'}\|^2\}_{\tau\leq t, s'\in\calS}$ that comes from all other terms we discuss in Steps 3-5. 

\paragraph{Obtaining average duality-gap bound} To obtain the average duality-gap bound in \pref{thm: main 1}, we sum \pref{eq: aggre} over $t$, and further argue that the sum of $\term_5$ over $t$ is smaller than the sum of $\term_6$ over $t$ (hence they are canceled with each other). Rearranging and telescoping leads to 
\begin{align*}
    \eta^2\sum_{t=1}^T \sum_s \Delta^2(\zp_{t+1}^s) = \order\left( \sum_{t=1}^T  \sum_{\tau=1}^t \sum_{s}\mu_\tau^{s}\alpha_{\tau-1}\dist_{\star}^2(\zp_\tau^{s}) \right) = \order\left( \sum_{t=1}^T  \sum_{\tau=1}^t \sum_{s}\mu_\tau^{s}\alpha_{\tau-1} \right). 
\end{align*}
As long as $\alpha_t$ is decreasing and going to zero, the right-hand side above can be shown to be sub-linear in $T$. Further relating $\max_{x',y'}\left(V^s_{\xp_t,y'} - V^s_{x',\yp_t}\right)$ to $\Delta(\zp_{t}^s)$ (\pref{lem: relate duality gap}) proves \pref{thm: main 1}.  

\paragraph{Obtaining last-iterate convergence bound}
Following the matrix game case, there is a problem-dependent constant $C>0$ such that
$\Delta(\zp_{t+1}^s) \geq C\dist_\star(\zp_{t+1}^s)$.
Similarly to how \pref{eq: matrix game last-iterate} is obtained, we use this in \pref{eq: aggre} and arrive at
\begin{align}
    &\sum_{s}\dist_\star^2(\zp_{t+1}^s)
    \leq \frac{1}{1+\Omega(\eta^2 C^2)} \sum_s \dist_{\star}^2(\zp_t^s) + \underbrace{\sum_{\tau=1}^t \sum_{s}\mu_\tau^{s}\alpha_{\tau-1}\dist_{\star}^2(\zp_\tau^{s})}_{\term_4} \nonumber \\
    &\qquad \qquad \qquad + \underbrace{\sum_{\tau=1}^t \sum_{s}\nu_\tau^{s}\|z_\tau^{s}-z_{\tau-1}^{s}\|^2}_{\term_5} 
    -  \Omega\Bigg(\underbrace{\sum_s\left(\norm{\zp_{t+1}^s-z_t^s}^2 + \norm{z_t^s - \zp_{t}^s}^2\right)}_{\term_6} \Bigg) \label{eq: aggre last iterate}
\end{align}
Then ideally we would like to follow a similar argument from \pref{eq: matrix game last-iterate} to \pref{eq: linear lastiterate} to obtain a last-iterate convergence guarantee. However, we face two more challenges here. 
First, we have an extra $\term_4$. 
Fortunately, this term vanishes when $t$ is large as long as $\alpha_{t}$ decreases and converges to zero. 
Second, in \pref{eq: matrix game last-iterate}, the indices of the negative term $\|\zp_{t+1}-z_t\|^2 + \|z_t-\zp_t\|^2$ and the positive term $\eta^2 \|z_t-z_{t-1}\|^2$ are only offset by $1$ so that a simple rearrangement is enough to get \pref{eq: linear lastiterate},
while in \pref{eq: aggre last iterate}, the indices in $\term_6$ and $\term_5$ are far from each other. To address this issue, we further introduce a set of weights and consider a weighted sum of \pref{eq: aggre last iterate} over $t$. We then show that the weighted sum of $\term_5$ can be canceled by the weighted sum of $\term_6$.  
Combining the above proves \pref{thm: main 2}.
Note that due to these extra terms, our last-iterate convergence rate is only sublinear (while \pref{eq: matrix game last itera} shows a linear rate for matrix games).



\section{Conclusion and Future Directions}
In this work, we propose the first decentralized algorithm for two-player zero-sum Markov games that is rational, convergent, agnostic, symmetric, and having a finite-time convergence rate guarantee at the same time. The algorithm is based on running OGDA on each state, together with a slowly changing critic that stabilizes the game matrix on each state. 

Our work studies the most basic tabular setting, and also requires a structural assumption when estimation is needed that sidesteps the difficulty of performing exploration over the state space. Important future directions include relaxing either of these assumptions, that is, extending our framework to allow function approximation and/or incorporating efficient exploration mechanisms. Studying OGDA-based algorithms beyond the two-player zero-sum setting is also an interesting future direction. 

\acks{
This work is supported by NSF Award IIS-1943607 and a Google Faculty Research Award.
}

\bibliography{ref}

\appendix

\section{Notations}\label{app: notation}
\subsection{Simplifications of the Notations}
We define the following notations to simplify the proofs: 
\begin{definition}
    $\xp^s_0=x^s_0=\mathbf{0}_{|\calA|}$ (zero vector with dimension $|\calA|$), $\yp^s_0=y^s_0=\mathbf{0}_{|\calB|}$, $Q^s_0=\mathbf{0}_{|\calA|\times |\calB|}$, $\ell^s_0=\mathbf{0}_{|\calA|}$, $r^s_0=\mathbf{0}_{|\calB|}$, $\pay^s_0=0$, $\alpha_0=1$.     
\end{definition}
Besides, for a matrix $Q$, we define $\norm{Q}=\max_{i,j}|Q_{ij}|$. 
To avoid cluttered notation, a product of the form $x^\top Qy$ is usually simply written as $xQy$. 

\subsection{Auxiliary Coefficients}
\label{app: aux coeff}
In this subsection, we define several coefficients that are related to the value learning rate $\{\alpha_t\}$. 
\begin{definition}
    \emphdef{$\alpha^\tau_t$}
    For non-negative integers $\tau$ and $t$ with $\tau\leq t$, define $\alpha^\tau_t = \alpha_\tau \prod_{i=\tau+1}^t (1-\alpha_i)$. 
\end{definition}

\begin{definition}
    \emphdef{$\delta^\tau_t$}  
    For non-negative integers $\tau$ and $t$ with $\tau\leq t$, define 
    $\delta^\tau_{t}\triangleq \prod_{i=\tau+1}^{t}(1-\alpha_i)$. 
\end{definition}

\begin{definition}
    \emphdef{$\beta^\tau_t$}
    For positive integers $\tau$ and $t$ with $\tau < t$, define $\beta^\tau_t = \alpha_\tau\prod_{i=\tau}^{t-1}(1-\alpha_i+\alpha_i\gamma)$. Define $\beta^t_t=1$. 
\end{definition}

\begin{definition}
    \emphdef{$\lambda_t$} For positive integers $t$, define $\lambda_t = \max\left\{\frac{\alpha_{t+1}}{\alpha_t}, 1-\frac{\alpha_t(1-\gamma)}{2}\right\}$.
\end{definition}

\begin{definition}
    \emphdef{$\lambda^\tau_t$}
    For positive integers $\tau$ and $t$ with $\tau < t$, define $\lambda^\tau_t = \alpha_\tau\prod_{i=\tau}^{t-1}\lambda_i$.  Define $\lambda^t_t=1$. 
\end{definition}

\subsection{Auxiliary Variables}
In this subsection, we define several auxiliary variables to be used in the later analysis. 
\begin{definition}
    \label{def: cumm path length}
    \emphdef{$\difz^s_t$}
    For every state $s\in\calS$, define the sequence $\{\difz_t^s\}_{t=1, 2, \ldots}$ by
    \begin{align*}
        \difz^s_1 &= \norm{z_1^s-z_0^s}^2, \\
        \difz^s_t &= (1-\alpha_t)\difz_{t-1}^s + \alpha_t \norm{\z_t^s - \z_{t-1}^s}^2, \qquad \forall t\geq 2.
    \end{align*}
    Furthermore, define $\difz_t\triangleq \max_{s} \difz_t^s$. 
\end{definition}

\begin{definition}
    \label{def: cumm Q diff}
    \emphdef{$\difq^s_t$}
    For every state $s\in\calS$, define the sequence $\{\difq_t^s\}_{t=1, 2, \ldots}$ by
    \begin{align*}
        \difq^s_1 &= \norm{Q_1^s - Q_0^s}^2, \\
        \difq^s_t &= (1-\alpha_t)\difq_{t-1}^s + \alpha_t \norm{\Q_t^s - \Q_{t-1}^s}^2, \qquad \forall t\geq 2.
    \end{align*}
    Furthermore, define $\difq_t\triangleq \max_{s} \difq_t^s$. 
\end{definition}

\begin{definition}
    \emphdef{$\xp_{t\star}^s, \yp_{t\star}^s, \zp_{t\star}^s$} Define $\xp_{t\star}^s=\Pi_{\calX_\star^s}(\xp_t^s)$, i.e., the projection of $\xp_t^s$ onto the set of optimal policy $\calX_\star^s$ on state $s$. Similarly, $\yp_{t\star}^s=\Pi_{\calY_\star^s}(\yp_t^s)$, and $\zp_{t\star}^s=\Pi_{\calZ_\star^s}(\zp_t^s) = (\xp_{t\star}^s, \yp_{t\star}^s)$. 
\end{definition}

\begin{definition}
    \label{def: duality t}
    \emphdef{$\Delta^s_t$}
    Define $\Delta^s_t=\max_{x', y'} \left(\xp_t^s Q_\star^s y'^s - x'^s Q_\star^s \yp_t^s \right)$ for all $t\geq 1$.  
\end{definition}

\begin{definition}
    \emphdef{$\barReg^s_t$}
    Define 
    \begin{align*}
        \barReg^s_t = \max\left\{\sum_{\tau=1}^{t}\alpha^\tau_{t} (x_{\tau}^{s}-\xp_{t\star}^{s}) Q_{\tau}^{s} y_{\tau}^{s}, \quad  \sum_{\tau=1}^{t}\alpha^\tau_{t} x_{\tau}^{s} Q_{\tau}^{s} (\yp_{t\star}^s - y_{\tau}^{s}) \right\}
    \end{align*}
    and $\barReg_t = \max_s \barReg_t^s$. 
\end{definition}

\begin{definition}\label{def: A recursive}
\emphdef{$\gap_t$}
Define $\gap_t=\max_{s}\norm{Q_t^s-Q_\star^s}$. 

\end{definition}

\begin{definition}
    \emphdef{$\theta^s_t$} Define $\theta_t^s = \frac{1}{16}\|\zp_{t}^s - z_{t-1}^s\|^2 + \frac{1}{16}\|z_{t-1}^s-\zp_{t-1}^s\|^2$
\end{definition}

\begin{definition}
    \emphdef{$Z_t$}
     Define 
$    Z_{t} = \max_s \sum_{\tau=1}^{t} \alpha^\tau_t\alpha_{\tau-1}  \dist_{\star}(\zp_\tau^s) 
$.  
\end{definition}


\subsection{\texorpdfstring{Assumptions on $\alpha_t$ and Simple Facts about $\alpha^\tau_t$}{}}
We require $\alpha_t$ to satisfy the following: 
\begin{itemize}
     \item $\alpha_1= 1$ 
     \item $0<\alpha_{t+1} \leq \alpha_t \leq 1$ 
     \item $\alpha_t \rightarrow 0$ as $t\rightarrow \infty$
\end{itemize} 
Furthermore, $\alpha_0\triangleq 1$. Below is an useful lemma that is used in many places: 
\begin{lemma}\label{lemma: recur1}
    If $\{h_t\}_{t=0,1,2,\ldots}$ and $\{k_t\}_{t=1,2,\ldots}$ are non-negative sequences that satisfy 
    $h_t = (1-\alpha_t)h_{t-1} + \alpha_t k_t$
    for $t\geq 1$, then $h_t = \sum_{\tau = 1}^{t}\alpha^{\tau}_t k_\tau$.   
\end{lemma}

\begin{proof}
    We prove it by induction. When $t=1$, since $\alpha_1=1$, $h_1= k_1 = \alpha^1_1k_1$. Assume that the formula is correct for $h_t$. Then
    \begin{align*} 
        h_{t+1} &= (1-\alpha_{t+1})h_t + \alpha_{t+1}k_{t+1} \\
        &= (1-\alpha_{t+1})\sum_{\tau=1}^t \alpha^\tau_t k_\tau + \alpha_{t+1}^{t+1}k_{t+1}\\
        &= \sum_{\tau=1}^t \alpha^\tau_{t+1} k_\tau+ \alpha_{t+1}^{t+1}k_{t+1} 
        = \sum_{\tau=1}^{t+1} \alpha^\tau_{t+1} k_\tau. 
    \end{align*}
\end{proof}

\begin{corollary}
    \label{cor: useful corollary}
    The following hold: 
    \begin{itemize}
        \item $V^s_t = \sum_{\tau=1}^t \alpha^\tau_t \pay^s_\tau$
        \item $\difz_t^s = \sum_{\tau=1}^t \alpha^\tau_t \norm{z_\tau^s-z_{\tau-1}^s}^2$ 
        \item $\difq_t^s = \sum_{\tau=1}^t \alpha^\tau_t \norm{Q_\tau^s - Q_{\tau-1}^s}^2$
    \end{itemize}
\end{corollary}

\begin{proof}
    They immediately follow from \pref{lemma: recur1} and the definition of $\difz_t^s, \difq_t^s, V^s_t$. 
\end{proof}

\section{Proof for Step 1: Single-Step Inequality}\label{app: step-1}
\begin{lemma} \label{lem: one-step regret bound}
    For any state $s$ and $t$, 
    \begin{align*}
        \left(x_t^s - \xp_{t\star}^s\right) Q^s_t y^s_{t} &\leq \frac{1}{2\eta}\Big(\dist_{\star}^2(\xp^s_t) - \dist_{\star}^2(\xp_{t+1}^s) -  \|\xp_{t+1}^s-x_t^s\|^2 - \|x_t^s-\xp_t^s\|^2\Big) \\
        &\qquad \qquad \qquad \qquad + \frac{4\eta}{(1-\gamma)^2} \|y^s_{t} - y^s_{t-1}\|^2 + 4\eta \|Q^s_{t} - Q^s_{t-1}\|^2 + 3\err, \\ 
         x_t^s Q^s_t (\yp_{t\star}^s -  y^s_{t}) &\leq \frac{1}{2\eta}\Big(\dist_{\star}^2(\yp_t^s) - \dist_{\star}^2(\yp_{t+1}^s)- \|\yp_{t+1}^s-y_t^s\|^2 - \|y_t^s-\yp_t^s\|^2\Big) \\
        &\qquad \qquad \qquad \qquad + \frac{4\eta}{(1-\gamma)^2} \|x^s_{t} - x^s_{t-1}\|^2 + 4\eta \|Q^s_{t} - Q^s_{t-1}\|^2 + 3\err. 
    \end{align*}
\end{lemma}

\begin{proof}
    By standard proof of OGDA (see, e.g., the proof of Lemma 1 in \citep{wei2020linear} or Lemma 1 in \citep{rakhlin2013optimization}), we have 
    \begin{align*}
        \left(x_t^s - \xp_{t\star}^s\right)^\top \ell^s_t &\leq \frac{1}{2\eta}\Big(\norm{\xp_t^s - \xp_{t\star}^s}^2 - \norm{\xp_{t+1}^s - \xp_{t\star}^s}^2  -  \|\xp_{t+1}^s-x_t^s\|^2 - \|x_t^s-\xp_t^s\|^2\Big) + \eta \|\ell_t^s-\ell_{t-1}^s\|^2.   
    \end{align*}
    Since $\norm{\xp_{t}^s - \xp_{t\star}^s}^2 = \dist^2_\star(\xp_{t}^s)$ and $\norm{\xp_{t+1}^s - \xp_{t\star}^s}^2 \geq \dist^2_\star(\xp_{t+1}^s)$ by the definition of $\dist_{\star}(\cdot)$, we further have
    \begin{align}
        \left(x_t^s - \xp_{t\star}^s\right)^\top \ell^s_t &\leq \frac{1}{2\eta}\Big(\dist^2_\star(\xp_t^s) - \dist^2_\star(\xp_{t+1}^s)  -  \|\xp_{t+1}^s-x_t^s\|^2 - \|x_t^s-\xp_t^s\|^2\Big) + \eta \|\ell_t^s-\ell_{t-1}^s\|^2.  \label{eq: x-regret}
    \end{align}
    By the definition of $\ell_t^s$, we have
    \begin{align*}
        &\eta \norm{\ell_t^s-\ell_{t-1}^s}^2 \\
        &\leq \eta \norm{\ell_t^s - Q_t^sy_t^s + (Q_t^s - Q_{t-1}^s)y_t^s + Q_{t-1}^s(y_t^s - y_{t-1}^s) + Q_{t-1}^s y_{t-1}^s - \ell_{t-1}^s}^2 \\
        &\leq 4\eta \norm{\ell_t^s - Q_t^s y_t^s}^2 + 4\eta \norm{(Q_t^s - Q_{t-1}^s)y_t^s }^2 +
        4\eta \norm{Q_{t-1}^s(y_t^s - y_{t-1}^s)}^2 + 
        4\eta\norm{Q_{t-1}^s y_{t-1}^s - \ell_{t-1}^s}^2 \\
        &\leq 4\eta \norm{Q_t^s-Q_{t-1}^s}^2 + \frac{4\eta}{(1-\gamma)^2}\norm{y_t^s-y_{t-1}^s}^2+ 8\eta\err^2
    \end{align*}
    and 
    \begin{align*}
        \left(x_t^s - \xp_{t\star}^s\right) Q_t^s y_t^s \leq \left(x_t^s - \xp_{t\star}^s\right) \ell^s_t  + 2\err. 
    \end{align*}
    Combining them with \pref{eq: x-regret} and the fact that $\eta\err\leq \frac{\eta}{1-\gamma}\leq \frac{1}{8}$, we get 
    the first inequality that we want to prove. The other inequality is similar. 
\end{proof}

\begin{lemma}
     \label{lemma: distance decrease}
     For all $t\geq 1$, 
     \begin{align*}
         \dist_{\star}^2(\zp_{t+1}^s) \leq \dist_{\star}^2(\zp_{t}^s) - 15\theta^s_{t+1} + \theta^s_t + 4\eta \gap_t + 8\eta^2\|Q_t^s-Q_{t-1}^s\|^2 + 6\eta\err.  
     \end{align*}
\end{lemma}
\begin{proof}
    Summing up the two inequalities in \pref{lem: one-step regret bound}, we get 
    \begin{align*}
        & 2\eta(x_t^s-\xp_{t\star}^s)Q_t^s y_t^s + 2\eta x_t^s Q_t^s (\yp_{t\star}^s-y_t^s) \\
        & \le \dist_{\star}^2(\zp_t^s)-\dist_{\star}^2(\zp_{t+1}^s)+\frac{4\eta^2}{(1-\gamma)^2}\|z_t^s-z_{t-1}^s\|^2 + 8\eta^2\|Q_t^s-Q_{t-1}^s\|^2-\|\zp_{t+1}^s-z_t^s\|^2-\|z_t^s-\zp_t^s\|^2 + 6\eta\err\\
        &\le \dist_{\star}^2(\zp_t^s)-\dist_{\star}^2(\zp_{t+1}^s)+\frac{1}{32}\|z_t^s-z_{t-1}^s\|^2 + 8\eta^2\|Q_t^s-Q_{t-1}^s\|^2-\|\zp_{t+1}^s-z_t^s\|^2-\|z_t^s-\zp_t^s\|^2 + 6\eta\err\\ 
        &\le \dist_{\star}^2(\zp_t^s)-\dist_{\star}^2(\zp_{t+1}^s)+\frac{1}{16}\left(\|z_t^s-\zp_{t}^s\|^2+\|\zp_t^s-z_{t-1}^s\|^2\right) \\
        &\qquad \qquad \qquad + 8\eta^2\|Q_t^s-Q_{t-1}^s\|^2-\|\zp_{t+1}^s-z_t^s\|^2-\|z_t^s-\zp_t^s\|^2 + 6\eta\err\\
        &= \dist_{\star}^2(\zp_t^s) - \dist_{\star}^2(\zp_{t+1}^s) + 8\eta^2\|Q_t^s-Q_{t-1}^s\|^2 - \frac{15}{16}\|z_t^s-\zp_t^s\|^2 - \|\zp_{t+1}^s-z_t^s\|^2 + \frac{1}{16}\| \zp_t^s-z_{t-1}^s\|^2  + 6\eta\err
    \end{align*}
    The left-hand side above, can be lower bounded by 
    \begin{align*}
        2\eta (x_t^s-\xp_{t\star}^s)Q_t^s y_t^s + 2\eta x_t^s Q_t^s (\yp_{t\star}^s-y_t^s) 
        &= 2\eta x_t^s Q^s_t \yp_{t\star}^s -2\eta \xp_{t\star}^s Q^s_t y_t^s \\
        &\geq 
        2\eta x_t^s Q^s_\star \yp_{t\star}^s - 2\eta \xp_{t\star}^s Q^s_\star y_t^s - 4\eta \gap_t\\
        &\geq - 4\eta \gap_t. \tag{by the optimality of $\xp_{t\star}^s$ and $\yp_{t\star}^s$}
    \end{align*} 
    Combining the inequalities and using the definition of $\theta^s_t$ finish the proof. 
%
\end{proof}

\section{\texorpdfstring{Proof for Step 2: Lower Bounding $\|\zp_{t+1}^s-z_t^s\|^2+\|z_t^s-\zp_t^s\|^2$}{}}\label{app: step-2}
\begin{lemma}\label{lem: step-2}
    For all $t\geq 1$, we have $20\theta_{t+1}^s+\eta\gap_t+2\eta^2\epsilon^2\geq\frac{\eta^2}{64}\left(\Delta_{t+1}^s\right)^2$.
\end{lemma}

\begin{proof}
By \pref{eq: update 1} and the optimality condition for $\xp_{t+1}^s$, we have 
    \begin{align}
        (\xp_{t+1}^s-\xp_{t}^s + \eta \ell_t^s)\cdot (x'^s-\xp_{t+1}^s) \geq 0 \label{eq: temp100}
    \end{align}
    for any $x'^s\in \Delta_{\calA}$. Then
    by the definition of $\ell^s_t$, 
    \begin{align}
        (\xp_{t+1}^s-\xp_{t}^s + \eta Q_{t}^sy_{t}^s)\cdot (x'^s-\xp_{t+1}^s) \geq 
        (\xp_{t+1}^s-\xp_{t}^s + \eta \ell_t^s)\cdot (x'^s-\xp_{t+1}^s) - 2\eta\err \geq -2\eta\err\label{eq: temp200}
    \end{align}
    where in the last inequality we use \pref{eq: temp100}. 
     Thus we have for any $x'^s\in \Delta_{\calA}$, 
    \begin{align*}
        &\sqrt{2}(\|\xp_{t+1}^s-x_t^s\|+\|x_t^s-\xp_t^s\|) \\
        &\geq \sqrt{2}\|\xp_{t+1}^s-\xp_t^s\|\\
        &\geq \|\xp_{t+1}^s-\xp_t^s\|^2\\
        &\geq
        (\xp_{t+1}^s-\xp_{t}^s)\cdot (x'^s-\xp_{t+1}^s)\\ 
        &\geq \eta (\xp_{t+1}^s-x'^s) Q_{t}^s y_{t}^s - 2\eta\err \tag{by \pref{eq: temp200}}\\
        &= \eta (x_{t}^s-x'^s) Q_{t}^s y_{t}^s + \eta(\xp_{t+1}^s-x_t^s)Q_t^sy_t^s - 2\eta\err\\
        &\geq \eta (x_{t}^s-x'^s) Q_{t}^s y_{t}^s - \frac{\eta \|\xp_{t+1}^s - x_{t}^s\|}{1-\gamma}- 2\eta\err.
    \end{align*}
    Using the fact that $\frac{\eta}{1-\gamma} \le \frac{1}{16}$, we get 
    \begin{align*}
        \|\xp_{t+1}^s - x_{t}^s\| + \|x_{t}^s - \xp_{t}^s\| + \sqrt{2}\eta\err \geq \frac{1}{\sqrt{2}+\frac{1}{16}}\left(\eta \max_{x'}(x_{t}^s-x'^s) Q_{t}^s y_{t}^s\right)\geq \frac{\eta}{2} \max_{x'}(x_{t}^s-x'^s) Q_{t}^s y_{t}^s. 
    \end{align*}
    Similarly, we have $\|\yp_{t+1}^s-y_t^s\|+\|y_t^s-\yp_t^s\| + \sqrt{2}\eta\err \ge \frac{\eta}{2}\max_{y'}x_t^s Q_t^s(y'^s-y_t^s)$.
    Combining them and using $\norm{z-z'}\geq \frac{1}{2}\norm{x-x'} + \frac{1}{2}\norm{y-y'}$, we get 
    \begingroup
    \allowdisplaybreaks
    \begin{align}
        &\|\zp_{t+1}^s - z_{t}^s\| + \|z_{t}^s - \zp_{t}^s\| + \sqrt{2}\eta\err \nonumber  \\  
        &\geq \frac{\eta}{4}\left(  \max_{x'}(x_{t}^s-x'^s) Q_{t}^s y_{t}^s + \max_{y'}x_{t}^{s} Q_{t}^s (y'^s-y_{t}^s) \right) \nonumber \\
        &\geq \frac{\eta}{4} \Big( \max_{y'} x_{t}^{s}Q_{t}^sy'^s - \min_{x'} x'^sQ_{t}^sy_{t}^s \Big) \nonumber \\
        &= \frac{\eta}{4}\max_{y'}\Big(\xp_{t+1}^sQ_\star^sy'^s+x_t^s(Q_t^s-Q_\star^s)y'^s+(x_t^s-\xp_{t+1}^s)Q_\star^sy'^s\Big)   \nonumber\\
        &\qquad \qquad -\frac{\eta}{4}\min_{x'}\Big(x'^sQ_\star^s\yp_{t+1}^s+x'^s(Q_t^s-Q_\star^s)y_t^s+x'Q_\star^s(y_t^s-\yp_{t+1}^s)\Big)   \nonumber \\
        &\geq \frac{\eta}{4} \max_{x',y'}\Big( \xp_{t+1}^{s}Q_{\star}^s y'^s - x'^s Q_{\star}^s\yp_{t+1}^s \Big) - \frac{\eta \gap_t}{2} - \frac{\eta}{4(1-\gamma)}\left(\|\xp_{t+1}^s-x_t^s\|+\|\yp_{t+1}^s-y_t^s\|\right)  \tag{$\|Q^s_\star\|\leq \frac{1}{1-\gamma}$}  \\
        &\geq \frac{\eta}{4} \Delta_{t+1}^s - \frac{\eta \gap_t}{2} - \frac{\eta}{2(1-\gamma)}\|\zp_{t+1}^s-z_t^s\| \tag{by the definition of $\Delta^s_{t+1}$} \\
        &\geq \frac{\eta}{4} \Delta_{t+1}^s - \frac{\eta \gap_t}{2} - \frac{1}{16}\|\zp_{t+1}^s-z_t^s\|.   \label{eq: standard optimality analysis}
    \end{align}
    \endgroup 
    Then notice that we have 
    \begin{align*}
        &20\theta^s_{t+1} + \eta \gap_t + 2\eta^2 \err^2  \\
        &\geq \frac{289}{256}\|\zp_{t+1}^s - z_{t}^s\|^2 + \|z_{t}^s - \zp_{t}^s\|^2 + \frac{\eta^2 \gap_t^2}{4} + 2\eta^2 \err^2   \tag{by the definition of $\theta^s_{t+1}$ and that $\eta\gap_t \leq \frac{\eta}{1-\gamma}\leq 1$}    \\
        &\geq \frac{1}{4}\left(\frac{17}{16}\|\zp_{t+1}^s - z_{t}^s\| + \|z_{t}^s - \zp_{t}^s\| + \frac{\eta \gap_t}{2} + \sqrt{2}\eta\err \right)^2    \tag{Cauchy-Schwarz inequality}\\
        &\geq \frac{\eta^2}{64} \left( \Delta_{t+1}^s \right)^2 \tag{by \pref{eq: standard optimality analysis} and notice that $\Delta_{t+1}^s\geq 0$}.
    \end{align*}
\end{proof}

\begin{lemma}
\label{theorem: main theorem discount}
\textbf{\emph{(Key Lemma for Average Duality-gap Bounds)}\ \ } 
For all $t\geq 1$, we have
\begin{align*}
    \dist_{\star}^2(\zp_{t+1}^s) 
        &\leq \dist_{\star}^2(\zp_t^s) - 5\theta_{t+1}^s + \theta_t^s - \frac{\eta^2}{128} (\Delta_{t+1}^s)^2 + 5\eta \gap_t + 8\eta^2\|\Q_t^s-\Q_{t-1}^s\|^2 + 7\eta\err.  
\end{align*}
\end{lemma} 
\begin{proof}
    Combining \pref{lem: step-2} with \pref{lemma: distance decrease}, we get 
    \begin{align*}
        \dist_{\star}^2(\zp_{t+1}^s) 
        &\leq \dist_{\star}^2(\zp_{t}^s) - 5\theta^s_{t+1} - 10\theta^s_{t+1} + \theta^s_t 
         + 4\eta \gap_t + 8\eta^2\|Q_t^s-Q_{t-1}^s\|^2 + 6\eta\err \\
        &\leq \dist_{\star}^2(\zp_t^s) - 5\theta^s_{t+1} - \left(\frac{\eta^2}{128}\left(\Delta^s_{t+1}\right)^2 - \frac{1}{2}\eta\gap_t - \eta^2\err^2\right) + \theta^s_t 
         + 4\eta \gap_t + 8\eta^2\|Q_t^s-Q_{t-1}^s\|^2 + 6\eta\err\\
        &\leq \dist_{\star}^2(\zp_t^s) - 5\theta^s_{t+1} + \theta^s_t  - \frac{\eta^2}{128}\left(\Delta^s_{t+1}\right)^2   
         + 5\eta \gap_t + 8\eta^2\|Q_t^s-Q_{t-1}^s\|^2 + 7\eta\err.   \tag{$\eta\err\leq 1$}
    \end{align*} 
    
\end{proof}

\begin{lemma}
    \label{lemma: using SPRSI}
    \textbf{\emph{(Key Lemma for Point-wise Convergence Bounds)}\ \ }
    There exists a constant $C'>0$ (which depends on the transition and the loss/payoff functions) such that for all $t\geq 1$, 
    \begin{align*}
        \dist_{\star}(\zp_{t+1}^s) + 4.5\theta_{t+1}^s
        &\leq  \frac{1}{1+\eta^2 C'^2}(\dist_{\star}(\zp_t^s) + 4.5\theta_{t}^s) + 5\eta \gap_t + 8\eta^2\|Q_t^s-Q_{t-1}^s\|^2 - 3\theta_{t}^s + 7\eta\err.  
\end{align*}
\end{lemma}

\begin{proof}
By Theorem 5 of \citep{wei2020linear} or Lemma 3 of \citep{gilpin2012first}, we have 
    \begin{align*}
        \Delta_{t+1}^s \geq C\dist_{\star}(\zp_{t+1}^s)
    \end{align*}
    for some problem-dependent constant $0<C\leq \frac{1}{1-\gamma}$ ($C$ depends on $\{Q^s_\star\}_{s}$). 
    Thus \pref{theorem: main theorem discount} implies
    \begin{align*}
        \dist_{\star}^2(\zp_{t+1}^s) + 5\theta_{t+1}^s 
        &\leq \dist_{\star}^2(\zp_{t}^s) + \theta_{t}^s - \frac{\eta^2 C^2}{128} \dist_{\star}^2(\zp_{t+1}^s) + 5\eta \gap_t + 8\eta^2\|Q_t^s-Q_{t-1}^s\|^2 + 7\eta\err. 
    \end{align*}
    By defining $C'^2 = \frac{C^2}{128}$, we further get
    \begin{align*}
    \dist_{\star}^2(\zp_{t+1}^s) + \frac{5}{1+\eta^2C'^2}\theta_{t+1}^s
        &\leq \frac{1}{1+\eta^2 C'^2} \left(\dist_{\star}^2(\zp_t^s) + \theta_{t}^s + 5\eta \gap_t + 8\eta^2\|Q_t^s-Q_{t-1}^s\|^2 + 7\eta\err \right) \\
        &\leq  \frac{1}{1+\eta^2 C'^2}(\dist_{\star}^2(\zp_t^s) + \theta_{t}^s) + 5\eta \gap_t + 8\eta^2\|Q_t^s-Q_{t-1}^s\|^2 + 7\eta\err.  
    \end{align*}
    Notice that $\frac{5}{1+\eta^2 C'^2}\geq \frac{5}{1+\frac{1}{16^2}\times\frac{1}{128}}\geq 4.5$. Thus we further have 
    \begin{align*}
        \dist_{\star}^2(\zp_{t+1}^s) + 4.5\theta_{t+1}^s
        &\leq  \frac{1}{1+\eta^2 C'^2}(\dist_{\star}^2(\zp_t^s) + 4.5\theta_{t}^s) + 5\eta \gap_t + 8\eta^2\|Q_t^s-Q_{t-1}^s\|^2 - 3\theta_{t}^s + 7\eta\err  
    \end{align*}
    where in the last inequality we use $\frac{1}{1+\eta^2C'^2} \leq \frac{4.5}{1+\eta^2C'^2}-3$ because $\eta^2C'^2\leq \frac{1}{16^2}\times \frac{1}{128}$. 
\end{proof}

\section{\texorpdfstring{Proof for Step 3: Bounding $\|Q_t^s-Q_{t-1}^s\|^2$}{}}\label{app: step-3}
\begin{lemma} 
\label{lem: key lemma 1}
     We have for $t\geq 2$ and all $s\in\calS$, 
     \begin{align*}
         \|\Q_{t}^s - \Q_{t-1}^s\|^2   \leq  \frac{8\gamma^2}{(1-\gamma)^3}\difz_{t-1} + \frac{2\gamma^2}{1+\gamma} \difq_{t-1} + \frac{16\gamma^2\err^2}{1-\gamma}. 
     \end{align*}
\end{lemma}
\begin{proof}
    It is equivalent to prove that for all $t\geq 1$, 
    \begin{align*}
         \|\Q_{t+1}^s -  \Q_{t}^s\|^2   \leq  \frac{8\gamma^2}{(1-\gamma)^3}\difz_{t} + \frac{2\gamma^2}{1+\gamma} \difq_{t} + \frac{16\gamma^2\err^2}{1-\gamma}. 
    \end{align*}
    By definition, 
    \[Q_t^s(a,b) = 
           \sigma(s,a,b) + \gamma \E_{s'\sim p(\cdot|s,a,b)} \left[V_{t-1}^{s'} \right],\]
    we have
    \begin{align}
        &\|\Q_{t+1}^s - \Q_{t}^s\|^2  
        = \max_{a,b} (Q_{t+1}^s(a,b) - Q_{t}^s(a,b))^2  
        \leq \gamma^2 \max_{s'}\left( V_{t}^{s'} - V_{t-1}^{s'} \right)^2    \label{eq: relate Q to V}
    \end{align}
    Now it suffices to upper bound $\left( V_{t}^{s} - V_{t-1}^{s} \right)^2$ for any $s$. 
    By \pref{cor: useful corollary}, we have $V_{t-1}^s = \sum_{\tau=1}^{t-1}\alpha^\tau_{t-1} \pay^s_\tau$. 
    Therefore, 
    \begin{align*}
        V_{t}^s - V_{t-1}^s 
        &= \alpha_t \left(\pay_t^s -V_{t-1}^s\right)\\
        &=\alpha_t\left(\pay_t^s - \sum_{\tau=0}^{t-1} \alpha^{\tau}_{t-1} \pay_\tau^s \right)  \\
        &= \alpha_t \left(\sum_{\tau = 0}^{t-1}\alpha_{t-1}^{\tau}\left(\pay_t^s -\pay_\tau^s\right) \right) \tag{because $\sum_{\tau=0}^{t-1}\alpha_{t-1}^\tau = 1$}
    \end{align*}
    In the following calculation, we omit the superscript $s$ for simplicity.  
    By defining $\diff_h \triangleq |\pay_h - \pay_{h-1}|$, we have
    \begingroup
    \allowdisplaybreaks
    \begin{align*}
        (V_{t} - V_{t-1})^2 
        &\leq (\alpha_t)^2 \left(\sum_{\tau=0}^{t-1}  \alpha^\tau_{t-1} (\pay_t - \pay_\tau) \right)^2 \\
        &\leq (\alpha_t)^2 \left(\sum_{\tau=0}^{t-1} \alpha^\tau_{t-1} \sum_{h=\tau+1}^t \left(\pay_h-\pay_{h-1}\right)  \right)^2 \\
        &\leq (\alpha_t)^2 \left(\sum_{\tau=0}^{t-1} \alpha^\tau_{t-1} \sum_{h=\tau+1}^t \diff_h \right)^2 \\
        &= (\alpha_t)^2 \left( \sum_{h=1}^{t}\sum_{\tau=0}^{h-1}\alpha^\tau_{t-1} \diff_h \right)^2 \\
        &\leq (\alpha_t)^2  \left(\sum_{h=1}^t \delta^{h-1}_{t-1} \diff_h \right)^2.  \tag{by \pref{lemma: delta formula}}
    \end{align*}
    \endgroup
    Then we continue: 
    \begingroup
    \allowdisplaybreaks
    \begin{align*}
        &(V_t - V_{t-1})^2 \\ 
        &\leq  (\alpha_t)^2\left(\sum_{h=1}^{t}\delta^{h-1}_{t-1} \diff_h\right)^2  \\
        &\leq (\alpha_t)^2 \left(\sum_{h=1}^t \delta^{h-1}_{t-1} \right)\left(\sum_{h=1}^{t} \delta^{h-1}_{t-1} \diff_h^2 \right)    \tag{Cauchy-Schwarz inequality}\\
        &\leq \left(\sum_{h=1}^t \alpha_h \delta^{h-1}_{t-1} \right)\left(\sum_{h=1}^{t} \alpha_h\delta^{h-1}_{t-1} \diff_h^2 \right)  \tag{$\alpha_t\leq \alpha_h$ for $h\leq t$}\\
        &\leq \sum_{\tau=1}^{t} \alpha^\tau_{t} \diff_\tau^2   \tag{note that $\alpha_h\delta^{h-1}_{t-1}=\alpha_h \prod_{\tau=h}^{t-1}(1-\alpha_\tau)\leq \alpha_h \prod_{\tau=h+1}^{t}(1-\alpha_\tau)= \alpha^h_{t}$}\\ 
        &= \sum_{\tau=1}^{t} \alpha^\tau_t \Big(\pay_t - x_\tau Q_\tau y_\tau + \x_\tau (\Q_\tau-\Q_{\tau-1}) \y_\tau + (\x_\tau - \x_{\tau-1}) \Q_{\tau-1} \y_\tau \\
        &\qquad \qquad \qquad + \x_{\tau-1} \Q_{\tau-1} (\y_\tau - \y_{\tau-1}) + x_{\tau-1}Q_{\tau-1}y_{\tau-1}-\pay_{t-1}\Big)^2 \\
        &\leq \sum_{\tau=1}^{t} \alpha^\tau_t \left(\frac{8\err^2}{1-\gamma} + \frac{2}{1+\gamma}\|\Q_\tau-\Q_{\tau-1}\|^2 + \frac{8}{(1-\gamma)^3}\|\x_\tau-\x_{\tau-1}\|^2 + \frac{8}{(1-\gamma)^3}\|\y_\tau-\y_{\tau-1}\|^2 +\frac{8\err^2}{1-\gamma} \right), 
    \end{align*}
    where we use $(a+b+c+d+e)^2\leq \frac{8}{1-\gamma}a^2 +
    \frac{2}{1+\gamma}b^2 + \frac{8}{1-\gamma}c^2 + \frac{8}{1-\gamma}d^2 + \frac{8}{1-\gamma}e^2$ which is due to Cauchy-Schwarz inequality.  
    \endgroup
    By \pref{lemma: recur1} and the definitions of $\difz^s_t$, $\difq^s_t$, $\difz_t$, $\difq_t$ in \pref{def: cumm path length} and \pref{def: cumm Q diff},
    \begin{align*}
         &\sum_{\tau=1}^t\alpha_t^\tau \|\Q_\tau^s-\Q_{\tau-1}^s\|^2 = \difq^s_t \leq \difq_t, \\
         &\sum_{\tau=1}^t\alpha_t^\tau\|\z_\tau^s - \z_{\tau-1}^s\|^2 \leq \difz^s_t \leq \difz_t. 
    \end{align*}
Combining them with the previous upper bound for $(V_t^s-V_{t-1}^s)^2$, we get
    \begin{align*}
        (V_t^s-V_{t-1}^s)^2 &\leq \frac{8}{(1-\gamma)^3}\difz_t +  \frac{2}{1+\gamma}\difq_t +  \frac{16\err^2}{1-\gamma} 
    \end{align*}
    for all $s$. 
    Further combining this with \pref{eq: relate Q to V}, we get
    \begin{align*}
        &\|\Q_{t+1}^s - \Q_{t}^s\|^2 
        \leq  \frac{8\gamma^2}{(1-\gamma)^3}\difz_{t} +  \frac{2\gamma^2}{1+\gamma} \difq_{t} + \frac{16\gamma^2\err^2}{1-\gamma} . 
    \end{align*}
\end{proof}

\section{\texorpdfstring{Proof for Steps 4 and 5: Bounding $\|Q_t^s-Q_\star^s\|$}{}}\label{app: step-4-5}


\begin{lemma}
    \label{lem: optimistic lemma discount}
    For all $t\geq 2$, 
    \begin{align*}
        \gap_t \leq \gamma \left(\sum_{\tau=1}^{t-1} \alpha^{\tau}_{t-1}\gap_\tau + \barReg_{t-1} + \err\right).
    \end{align*}
\end{lemma}
\begin{proof}
We proceed with
    \begingroup
    \allowdisplaybreaks
    \begin{align*} 
        &Q_{t+1}^s(a,b) \\ 
        &= \sig(s,a,b) + \gamma\E_{s'\sim p(\cdot|s,a,b)}\left[ V_{t}^{s'} \right]  \\
        &= \sig(s,a,b) + \gamma\E_{s'\sim p(\cdot|s,a,b)}\left[ \sum_{\tau=1}^t \alpha^\tau_t \pay^{s'}_\tau \right] \tag{\pref{cor: useful corollary}} \\
        &\leq \sig(s,a,b) + \gamma\E_{s'\sim p(\cdot|s,a,b)} \left[\sum_{\tau=1}^{t}\alpha^\tau_{t} x_{\tau}^{s'} Q_{\tau}^{s'} y_{\tau}^{s'} + \err \right]  \tag{by the definition of $\pay^{s'}_\tau$ and that $\sum_{\tau=1}^t \alpha^\tau_t=1$ for $t\geq 1$}\\
        &\leq \sig(s,a,b) + \gamma\E_{s'\sim p(\cdot|s,a,b)} \left[ \sum_{\tau=1}^{t} \alpha^\tau_{t}  \xp_{t\star}^{s'} Q_{\tau}^{s'}y_\tau^{s'} + \barReg_{t} + \err \right]    \tag{by the definition of $\barReg_t$}\\
        &\leq \sig(s,a,b) + \gamma\E_{s'\sim p(\cdot|s,a,b)}\left[  \sum_{\tau=1}^{t} \alpha^\tau_{t}\xp_{t\star}^{s'} Q_{\star}^{s'}y_\tau^{s'} + \sum_{\tau=1}^{t}  \alpha^{\tau}_{t}\gap_\tau + \barReg_{t}+  \err \right]  \tag{by the definition of $\gap_\tau$} \\
        &\leq \sig(s,a,b) + \gamma\E_{s'\sim p(\cdot|s,a,b)}\left[  \sum_{\tau=1}^{t} \alpha^\tau_{t}\xp_{t\star}^{s'} Q_{\star}^{s'} y_\star^{s'}\right] + \gamma\left(\sum_{\tau=1}^{t} \alpha^{\tau}_{t}\gap_\tau + \barReg_{t} + \err \right) \tag{by definition of $y_\star^{s'}$} \\
        &= \sig(s,a,b) + \gamma \E_{s'\sim p(\cdot|s,a,b)}\left[V^{s'}_{\star}\right] + \gamma\left(\sum_{\tau=1}^{t} \alpha^{\tau}_{t}\gap_\tau + \barReg_{t} + \err \right)   \tag{$\sum_{\tau=1}^t \alpha^\tau_t=1$ for $t\geq 1$} \\
        &= Q^{s}_\star(a,b) + \gamma\left(\sum_{\tau=1}^{t} \alpha^{\tau}_{t}\gap_\tau + \barReg_{t} + \err \right). 
    \end{align*} 
    Similarly, 
    \begin{align*}
        Q^s_{t+1}(a,b)\geq Q^s_\star(a,b) - \gamma\left(\sum_{\tau=1}^{t} \alpha^{\tau}_{t}\gap_\tau + \barReg_{t} + \err \right). 
    \end{align*}
    They jointly imply
    \begin{align*}
        \gap_{t+1}\leq \gamma\left(\sum_{\tau=1}^{t} \alpha^{\tau}_{t}\gap_\tau + \barReg_{t} + \err \right).
    \end{align*}
    \endgroup
\end{proof}

\begin{lemma}
    \label{lem: interval regret discount nonuniform}
    For any state $s$ and time $t\geq 1$, 
    \begin{align*}
        \barReg_t
        \leq \frac{1}{2\eta}Z_t + \frac{4\eta}{(1-\gamma)^2} \difz_{t}  + 4\eta \difq_t + 3\err.
    \end{align*}
\end{lemma}
\begin{proof}
    Summing the first bound in \pref{lem: one-step regret bound} over $\tau=1,\ldots, t$ with weights $\alpha^\tau_t$, and dropping negative terms $-\|\xp_{t+1}^s-x_t^s\|^2 - \|x_t^s - \xp_t^2\|^2$, we get
    \begin{align}
        &\sum_{\tau=1}^t \alpha^\tau_t \left(x_\tau^s - \xp_{\tau\star}^s\right) Q^s_\tau y^s_{\tau} \nonumber \\
        &\leq \sum_{\tau=1}^t \frac{\alpha^\tau_t}{2\eta} \left(\dist_{\star}^2(\xp_\tau^s) - \dist_{\star}^2(\xp_{\tau+1}^s)\right) 
         + \frac{4\eta}{(1-\gamma)^2} \sum_{\tau=1}^t  \alpha^\tau_t \|y_\tau^s-y_{\tau-1}^s\|^2 
           +  4\eta \sum_{\tau=1}^t \alpha^\tau_t \|Q_\tau^{s}-Q_{\tau-1}^s\|^2 + 3\err   \nonumber  \\
        &\leq  \frac{\alpha^1_t}{2\eta}\dist_{\star}^2(\xp_1^s) + \sum_{\tau=2}^{t} \frac{\alpha^{\tau}_t - \alpha^{\tau-1}_t}{2\eta}\dist_{\star}(\xp_{\tau}^s)   + \frac{4\eta}{(1-\gamma)^2} \sum_{\tau=1}^t  \alpha^\tau_t \|y_\tau^s-y_{\tau-1}^s\|^2 
           +  4\eta \sum_{\tau=1}^t \alpha^\tau_t \|Q_\tau^{s}-Q_{\tau-1}^s\|^2+ 3\err   \nonumber  \\
        &\leq \frac{\alpha^1_t}{2\eta}\dist_{\star}^2(\xp_1^s) + \sum_{\tau=2}^{t} \frac{\alpha^{\tau}_t - \alpha^{\tau-1}_t}{2\eta}\dist_{\star}(\xp_{\tau}^s) + \frac{4\eta}{(1-\gamma)^2} \difz^s_{t} + 4\eta \difq^s_t + 3\err.     \label{eq: temp 167} 
    \end{align}
    Observe that by definition, we have for $\tau\geq 2$, 
    \begin{align*}
        \alpha_t^{\tau} - \alpha_t^{\tau-1} 
        &= \alpha^\tau_t \left(1-\frac{\alpha_{\tau-1}(1-\alpha_\tau)}{\alpha_\tau}\right) = \alpha^\tau_t \times \frac{\alpha_\tau - \alpha_{\tau-1} + \alpha_{\tau-1}\alpha_\tau}{\alpha_\tau} \leq \alpha_{\tau-1}\alpha^\tau_t 
    \end{align*}
    where in the inequality we use $\alpha_\tau\leq \alpha_{\tau-1}$. 
    Using this in \pref{eq: temp 167}, we get
    \begin{align*}
        \sum_{\tau=1}^t \alpha^\tau_t \left(x_\tau^s - \xp_{\tau*}^s\right) Q^s_\tau y^s_{\tau} 
        &\leq \frac{\alpha^1_t}{2\eta}\dist_{\star}(\xp_1^s) + \frac{1}{2\eta}\sum_{\tau=2}^t \alpha^\tau_t \alpha_{\tau-1} \dist_{\star}(\xp_\tau^s)  + \frac{4\eta}{(1-\gamma)^2} \difz^s_{t}  + 4\eta \difq^s_t + 3\err\\
        &= \frac{1}{2\eta}\sum_{\tau=1}^t \alpha^\tau_t \alpha_{\tau-1} \dist_{\star}(\xp_\tau^s)  + \frac{4\eta}{(1-\gamma)^2} \difz^s_{t}  + 4\eta \difq^s_t + 3\err  \tag{recall that $\alpha_0=1$}
    \end{align*}
    Using $\difz_n^s\leq \difz_n, \difq_n^s\leq \difq_n$, and the definition of $Z_t, \barReg_t$ finishes the proof. 
\end{proof}

\section{Combining Lemmas to Show Last-iterate Convergence}\label{app: combining}
In this section, we provide proofs for \pref{thm: main 1} and \pref{thm: main 2}. To achieve so, we first prove \pref{lemma: ut bound} by combining the results in \pref{app: step-3} and \pref{app: step-4-5}. Then we combine \pref{theorem: main theorem discount}, \pref{lemma: using SPRSI}, and \pref{lemma: ut bound} to prove \pref{thm: main 1} and \pref{thm: main 2}.

\begin{lemma}\label{lemma: ut bound}
    For any $s$ and $t\geq 1$, 
    \begin{align*}
        &5\eta \gap_t + 8\eta^2 \norm{Q_t^s-Q^s_{t-1}}^2\\ 
        &\leq \max_{s'} \left(\frac{C_1\eta^2}{(1-\gamma)^4} \sum_{\tau=1}^{t-1} \beta^{\tau}_t (\theta_\tau^{s'} + \theta_{\tau+1}^{s'})  + C_2 \sum_{\tau=1}^{t-1}\beta^\tau_t \alpha_{\tau-1} \dist_{\star}^2(\zp_\tau^{s'})\right) + 80\beta^1_t  + \frac{80\eta \err}{(1-\gamma)^2}. 
    \end{align*}
    where $C_1=1152\times 80$ and $C_2=10$. 
\end{lemma}

\begin{proof}
By \pref{lem: key lemma 1}, for all $t\geq 2$, 
\begin{align}
         \eta^2\|\Q_{t}^s - \Q_{t-1}^s\|^2   \leq  \frac{8\eta^2 \gamma^2}{(1-\gamma)^3}\difz_{t-1} + \frac{2\eta^2\gamma^2}{1+\gamma} \difq_{t-1} + \frac{16 \eta^2 \err^2}{1-\gamma}. \label{eq: final 1}
\end{align}
By \pref{lem: optimistic lemma discount} and \pref{lem: interval regret discount nonuniform}, for all $t\geq 2$, 
\begin{align}
    \eta\gap_t \leq \gamma \sum_{\tau=1}^{t-1} \alpha^{\tau}_{t-1}\eta \gap_\tau + \frac{4\eta^2}{(1-\gamma)^2}\difz_{t-1} + 4\eta^2 \difq_{t-1} + \frac{1}{2}Z_{t-1} + 4\eta\err.   \label{eq: final 2}   
\end{align}
%
Now, multiply \pref{eq: final 2} with $\frac{1-\gamma}{16}$, and then add it to \pref{eq: final 1}. Then we get that for $t\geq 2$, 
\begin{align*}
    &\eta^2 \|Q^s_{t} - Q^s_{t-1}\|^2 + \frac{1-\gamma}{16}\eta\gap_t \\ 
    &\leq \gamma \sum_{\tau=1}^{t-1} \alpha^{\tau}_{t-1}\left(\frac{1-\gamma}{16}\eta \gap_\tau\right)  + \left( \frac{8\gamma^2}{(1-\gamma)^3} + \frac{1}{4(1-\gamma)}  \right)\eta^2\difz_{t-1} + \\
    &\tag*{$\displaystyle\left(\frac{2\gamma^2}{1+\gamma} + \frac{1-\gamma}{4}\right) \eta^2 \difq_{t-1} +  \frac{(1-\gamma)Z_{t-1}}{32} + \frac{\eta\err}{2}$}   \\
    &\leq \gamma \sum_{\tau=1}^{t-1} \alpha^{\tau}_{t-1}\left(\frac{1-\gamma}{16}\eta \gap_\tau\right)  + \frac{9}{(1-\gamma)^3} \eta^2\difz_{t-1} + \gamma \eta^2 \difq_{t-1} +  \frac{(1-\gamma)Z_{t-1}}{32} + \frac{\eta\err}{2} \tag{see explanation below}\\
    &\leq \gamma \sum_{\tau=1}^{t-1} \alpha^{\tau}_{t-1}\left(\frac{1-\gamma}{16}\eta \gap_\tau + \eta^2 \max_{s'}\|Q^{s'}_\tau - Q^{s'}_{\tau-1}\|^2\right) + \frac{9}{(1-\gamma)^3}\eta^2\difz_{t-1} +  \frac{(1-\gamma)Z_{t-1}}{32} + \frac{\eta\err}{2},  
\end{align*}
where in the second inequality we use that $\frac{2\gamma^2}{1+\gamma}+\frac{1-\gamma}{4}-\gamma = (1-\gamma)\left(\frac{1}{4}-\frac{\gamma}{1+\gamma}\right)\leq 0$ since $\gamma\geq \frac{1}{2}$, and in the last inequality, we use $\difq^s_{t-1}=\sum_{\tau=1}^{t-1}\alpha^\tau_{t-1}\|Q^s_{\tau}-Q^s_{\tau-1}\|^2$. 
Define the new variable
\begin{align*}
    u_t = \eta^2 \max_s \|Q^s_t-Q^s_{t-1}\|^2 + \frac{1-\gamma}{16}\eta \gap_t. 
\end{align*}
Then the above implies that for all $t\geq 2$, 
\begin{align}
    u_{t}  \leq \gamma\sum_{\tau=1}^{t-1} \alpha^\tau_{t-1}u_\tau + \frac{9}{(1-\gamma)^3}\eta^2\difz_{t-1} +  \frac{(1-\gamma)Z_{t-1}}{32} + \frac{\eta\err}{2}.   \label{eq: key 1}
\end{align}  
Observe that \pref{eq: key 1} is in the form of \pref{lemma: double recursion} with the following choices: 
\begin{align*}
    g_t &= u_t, \ \ \forall t\geq 1, \\
    h_t &= \begin{cases}
         u_t + \frac{\eta\err}{2} &\text{for\ } t=1 \\
         \frac{9\eta^2}{(1-\gamma)^3}\difz_{t-1} + \frac{(1-\gamma)}{32}Z_{t-1} + \frac{\eta\err}{2}  &\text{for\ } t\geq 2
    \end{cases}
\end{align*}
and get that for $t\geq 2$, 
\begin{align*}
    u_t 
    &\leq \frac{9\eta^2}{(1-\gamma)^3} \sum_{\tau=2}^{t} \beta^{\tau}_t \difz_{\tau-1} + \frac{1-\gamma}{32} \sum_{\tau=2}^{t}\beta^\tau_t Z_{\tau-1} + \beta^1_t u_1 + \frac{\eta \err}{2} \sum_{\tau=1}^t \beta^\tau_t \\
    &\leq \frac{9\eta^2}{(1-\gamma)^3} \sum_{\tau=2}^{t} \beta^{\tau}_t \difz_{\tau-1} + \frac{1-\gamma}{32} \sum_{\tau=2}^{t}\beta^\tau_t Z_{\tau-1} + (1-\gamma)\beta^1_t + \frac{\eta \err}{1-\gamma}  \tag{by \pref{lemma: sum1 of beta}}
\end{align*}
because $u_t\leq \frac{\eta^2}{(1-\gamma)^2} + \frac{1-\gamma}{16}\leq \frac{1-\gamma}{2}$. 
Further using \pref{lemma: triple recursion} on the first two terms on the right-hand side, and noticing that $\frac{1}{\gamma^2}\leq 4$, we further get that for $t\geq 2$, 
\begin{align}  
    u_t 
    &\leq \max_s \frac{36\eta^2}{(1-\gamma)^3} \sum_{\tau=1}^{t-1} \beta^{\tau}_t \norm{z_{\tau}^s-z_{\tau-1}^s}^2 + \frac{1-\gamma}{8}\sum_{\tau=1}^{t-1}\beta^\tau_t \alpha_{\tau-1}  \dist_{\star}^2(\zp_\tau^s) + (1-\gamma)\beta^1_t  + \frac{\eta \err}{1-\gamma}\nonumber   \\
    &\leq \max_s \frac{72 \eta^2}{(1-\gamma)^3}   \sum_{\tau=1}^{t-1} \beta^{\tau}_t ( \|z_\tau^s-\zp_{\tau}^s\|^2 + \|\zp_{\tau}^s - z_{\tau-1}^s\|^2) +\nonumber\\
    &\tag*{$\displaystyle\frac{1-\gamma}{8} \sum_{\tau=1}^{t-1}\beta^\tau_t \alpha_{\tau-1} \dist_{\star}^2(\zp_\tau^s)+ (1-\gamma)\beta^1_t    + \frac{\eta \err}{1-\gamma} $}\nonumber \\
    &\leq \max_s \frac{1152\eta^2}{(1-\gamma)^3} \sum_{\tau=1}^{t-1} \beta^{\tau}_t (\theta_{\tau+1}^s + \theta_\tau^s)  + \frac{1-\gamma}{8}\sum_{\tau=1}^{t-1}\beta^\tau_t \alpha_{\tau-1} \dist_{\star}^2(\zp_\tau^s)+ (1-\gamma)\beta^1_t  + \frac{\eta \err}{1-\gamma}.\label{eq: key 3}
\end{align}
    Finally, notice that according to the definition of $u_t$, we have $5\eta \gap_t + 8\eta^2 \norm{Q_t^s-Q^s_{t-1}}^2 \leq \frac{80}{1-\gamma}u_t$. Combining \pref{eq: key 3}, we finish the proof for case for $t\geq 2$. The case for $t=1$ is trivial since $5\eta \gap_t + 8\eta^2 \norm{Q_t^s-Q^s_{t-1}}^2 \leq 1\leq 80=80\beta^1_1$.  
\end{proof}

\begin{proof}\textbf{of \pref{thm: main 1}. }\ \ \ 
    Define $\constA(T)\triangleq 1+\sum_{t=1}^T \alpha_t$ and let $\constB$ be an upper bound of $\sum_{t=\tau}^\infty \beta^\tau_t$ for any $\tau$. With the choice of $\alpha_t$ specified in the theorem, we have $C_\alpha(T)=1+\sum_{t=1}^T \frac{H+1}{H+t}=\order\left(H\log T\right)=\order\left(\frac{\log T}{1-\gamma}\right)$. By \pref{lem: sum of beta 1}, we have $C_\beta\leq \frac{2}{1-\gamma}+3$. Define $S=|\calS|$. 

    Combining \pref{lemma: ut bound} and \pref{theorem: main theorem discount}, we get that for $t\geq 1$, 
    \begin{align*}
        \frac{\eta^2}{128} (\Delta_{t+1}^s)^2 
        &\leq \dist_{\star}^2(\zp_t^s) - \dist_{\star}^2(\zp_{t+1}^s) - 5\theta_{t+1}^s + \theta_t^s \\
        &\hspace{-10pt} + \max_{s'}\left( \frac{C_1\eta^2}{(1-\gamma)^4} \sum_{\tau=1}^{t-1} \beta^{\tau}_t (\theta_\tau^{s'}+\theta_{\tau+1}^{s'}) +  C_2\sum_{\tau=1}^{t-1}\beta^\tau_t \alpha_{\tau-1} \dist_{\star}^2(\zp_\tau^{s'})\right) + 80\beta^1_t + \frac{87\eta\err}{(1-\gamma)^2}.
    \end{align*}
    Summing the above over $s\in \calS$ and $t\in [T-1]$, and denoting $\Theta_t=\sum_s \theta_t^s $, we get 
    \begin{align} 
        \frac{\eta^2}{128}\sum_{t=1}^T\sum_s (\Delta_t^s)^2 
        &\leq \order(S) - \sum_{t=1}^T 4\Theta_{t} + \frac{C_1S\eta^2}{(1-\gamma)^4}\sum_{t=1}^T \sum_{\tau=1}^{t-1} \beta^\tau_t (\Theta_\tau+\Theta_{\tau+1}) \nonumber \\
        &\qquad + \order\left(S\sum_{t=1}^T   \sum_{\tau=1}^{t-1}\beta^\tau_t\alpha_{\tau-1} + S\sum_{t=1}^T \beta^1_t + \frac{S\eta\err T}{(1-\gamma)^2}\right)  \label{eq: tmp bound}
    \end{align}
    since $\dist^2_\star(\zp_\tau^s)=\order(1)$ and $\Theta_1=\order(S)$. 
    Notice that the following hold
    \begin{align*}
        \sum_{t=1}^T \sum_{\tau=1}^{t-1} \beta^\tau_t (\Theta_\tau + \Theta_{\tau+1}) 
        &\leq \sum_{\tau=1}^{T-1}  \sum_{t=\tau}^T \beta^\tau_t (\Theta_\tau + \Theta_{\tau+1}) 
        \leq 2\constB \sum_{\tau=1}^T \Theta_\tau, 
    \end{align*}
    \begin{align*}
        &\sum_{t=1}^T \sum_{\tau=1}^{t-1} \beta^\tau_t \alpha_{\tau-1} \leq \sum_{\tau=1}^{T}\sum_{t=\tau}^{T}\beta^\tau_t \alpha_{\tau-1} \leq \constB \sum_{\tau=1}^{T}\alpha_{\tau-1} = \constA(T)\constB,
    \end{align*}
    and 
$
        \sum_{t=1}^T \beta^1_t \leq \constB
$.
Combining these three inequalities with \pref{eq: tmp bound}, we get 
    \begin{align*}
        \sum_{t=1}^T \sum_s (\Delta^s_t)^2 
        &= \frac{128}{\eta^2}\sum_{t=1}^T \left(-4+2\constB \frac{C_1  S\eta^2}{(1-\gamma)^4}\right)\Theta_t  + \order\left(\frac{S  \constA(T)\constB}{\eta^2} + \frac{S\err T}{\eta(1-\gamma)^2}\right) \\
        &= \order\left(\frac{S\constA(T) \constB}{\eta^2} + \frac{S\err T}{\eta(1-\gamma)^2}\right). \tag{by our choice of $\eta$, we have $2\constB \frac{C_1S\eta^2}{(1-\gamma)^4}\leq \frac{10C_1S\eta^2}{(1-\gamma)^5}\leq 4$} 
    \end{align*}
    By Cauchy-Schwarz inequality, we further have
    \begin{align*}
        \sum_{t=1}^T \sum_s \Delta^s_t \leq \sqrt{ST}\left(\sum_{t=1}^T\sum_s  (\Delta^s_t)^2\right)^{\frac{1}{2}} = \order\left(\frac{S\sqrt{\constA(T)\constB T}}{\eta} + \frac{ST\sqrt{\epsilon}}{\sqrt{\eta}(1-\gamma)}\right).  
    \end{align*}
    Finally, by \pref{lem: relate duality gap}, we get 
    \begin{align*}
        \frac{1}{T}\sum_{t=1}^T \max_{s,x',y'} \left(V^s_{\xp_t, y'} -V^s_{x', \yp_t} \right) 
        &\leq \frac{2}{1-\gamma}\frac{1}{T}\sum_{t=1}^T \max_s \Delta^s_t = \order\left(\frac{S\sqrt{\constA(T)\constB }}{\eta(1-\gamma)\sqrt{T}} + \frac{S\sqrt{\epsilon}}{\sqrt{\eta}(1-\gamma)^2}\right)\\
        &= \order\left(\frac{S\sqrt{\log T}}{\eta(1-\gamma)^2\sqrt{T}} + \frac{S\sqrt{\epsilon}}{\sqrt{\eta}(1-\gamma)^2}\right).
    \end{align*}
\end{proof}

\begin{proof}\textbf{of \pref{thm: main 2}. }\ \ \ 
Combining \pref{lemma: using SPRSI} and \pref{lemma: ut bound}, we get that for all $t\geq 1$, 
\begin{align*}
    &\dist_{\star}^2(\zp_{t+1}^s) + 4.5\theta_{t+1}^s \\
    &\leq \frac{1}{1 + \eta^2 C'^{2}} \left(\dist_{\star}^2(\zp_{t}^s) + 4.5\theta_{t}^s\right) \\
    &\hspace{-2pt}  +  \max_{s'} \left(\frac{C_1\eta^2}{(1-\gamma)^4} \sum_{\tau=1}^{t-1} \beta^{\tau}_t (\theta_\tau^{s'} + \theta_{\tau+1}^{s'})  + C_2 \sum_{\tau=1}^{t-1}\beta^\tau_t \alpha_{\tau-1} \dist_{\star}^2(\zp_\tau^{s'})\right) + 80\beta^1_t - 3\theta^s_{t} + \frac{87\eta\err}{(1-\gamma)^2}. 
\end{align*}
Summing the above inequality over $s\in \calS$, and denoting $L_t=\sum_{s} \dist_{\star}^2(\zp_{t}^s)$, $\Theta_t=\sum_{s} \theta_t^s$, we get that for all $t$, 
\begin{align}
    L_{t+1}+4.5\Theta_{t+1}&\leq \frac{1}{1+\eta^2 C'^2} (L_t+4.5\Theta_t) + \frac{C_1S\eta^2}{(1-\gamma)^4}\sum_{\tau=1}^{t-1} \beta^\tau_t (\Theta_\tau + \Theta_{\tau+1}) +\nonumber \\
    & \qquad C_2S\sum_{\tau=1}^{t-1}\beta^\tau_t \alpha_{\tau-1} L_\tau + 80S\beta^1_t - 3\Theta_{t} + \frac{87S\eta\err}{(1-\gamma)^2}.  \label{eq: key 4}
\end{align} 

The key idea of the following analysis is to use the negative (bonus) term $- 3\Theta_{t}$ to cancel the positive (penalty) term $\frac{C_1S\eta^2}{(1-\gamma)^4}\sum_{\tau=1}^t \beta^\tau_t \Theta_\tau$. Since the time indices do not match, we perform \emph{smoothing over time} to help. Consider the following weighted sum of $L_\tau+4.5\Theta_\tau$ with weights $\lambda^\tau_{t+1}$: 
\begingroup
\allowdisplaybreaks
\begin{align}  
    &\sum_{\tau=2}^{t+1} \lambda^{\tau}_{t+1}(L_\tau+4.5\Theta_\tau) \nonumber \\
    &=\sum_{\tau=1}^t \lambda^{\tau+1}_{t+1} (L_{\tau+1}+4.5\Theta_{\tau+1}) \tag{re-indexing} \\
    &\leq \sum_{\tau=1}^t \lambda^{\tau}_{t} (L_{\tau+1}+4.5\Theta_{\tau+1})  \tag{\pref{lem: lambda ratio consecutive}}   \\
    &\leq \frac{1}{1+\eta^2C'^2}\sum_{\tau=1}^t \lambda^{\tau}_t (L_\tau + 4.5\Theta_\tau) + \frac{C_1S\eta^2}{(1-\gamma)^4}\sum_{\tau=1}^t\lambda^{\tau}_t\sum_{i=1}^{\tau-1}\beta^i_{\tau} (\Theta_i + \Theta_{i+1})   \nonumber  \\
    &\qquad \qquad + C_2 S\sum_{\tau=1}^{t} \lambda^{\tau}_t \sum_{i=1}^{\tau-1}\beta^i_\tau  \alpha_{i-1} L_i + 80S\sum_{\tau=1}^t \lambda^\tau_t\beta^1_\tau  -3\sum_{\tau=1}^t \lambda^{\tau}_t \Theta_{\tau} + \frac{87S\eta\err}{(1-\gamma)^2}\sum_{\tau=1}^t \lambda^\tau_t \nonumber  \\
    &\leq \frac{1}{1+\eta^2C'^2}\sum_{\tau=1}^t  \lambda^\tau_t (L_\tau+4.5\Theta_\tau) + \frac{C_1S\eta^2}{(1-\gamma)^4}\sum_{i=1}^{t-1}\left(\sum_{\tau=i}^t\lambda^\tau_t \beta^i_\tau\right) (\Theta_i + \Theta_{i+1}) \nonumber  \\
    &\qquad \qquad + C_2 S\sum_{i=1}^{t-1} \left(\sum_{\tau=i}^t \lambda^\tau_t \beta^i_\tau\right) \alpha_{i-1} L_i + 80S\sum_{\tau=1}^t \lambda^\tau_t\beta^1_\tau   -3\sum_{\tau=1}^t \lambda^\tau_t \Theta_{\tau} + \frac{87S\eta\err}{(1-\gamma)^2}\sum_{\tau=1}^t \lambda^\tau_t   \nonumber \\
    &\leq \frac{1}{1+\eta^2C'^2}\sum_{\tau=1}^t \lambda^\tau_t (L_\tau+4.5\Theta_\tau) + \frac{3C_1S\eta^2}{(1-\gamma)^5}\sum_{\tau=1}^{t-1} \lambda^\tau_t (\Theta_\tau + \Theta_{\tau+1})  \nonumber \\
    &\qquad \qquad + \frac{3C_2 S}{1-\gamma}\sum_{\tau=1}^{t-1} \lambda^\tau_t \alpha_{\tau-1} L_\tau + \frac{240S}{1-\gamma}\lambda^1_t -3\sum_{\tau=1}^t \lambda^\tau_t \Theta_{\tau} + \frac{87S\eta\err}{(1-\gamma)^2}\sum_{\tau=1}^t \lambda^\tau_t \tag{by \pref{lem: lambda beta lambda}}\\
    &\leq \frac{1}{1+\eta^2C'^2}\sum_{\tau=1}^t \lambda^\tau_t (L_\tau+4.5\Theta_\tau) + \frac{6C_1S\eta^2}{(1-\gamma)^5}\sum_{\tau=1}^{t} \lambda^\tau_t \Theta_\tau  \nonumber \\
    &\qquad \qquad + \frac{3C_2 S}{1-\gamma}\sum_{\tau=1}^{t-1} \lambda^\tau_t \alpha_{\tau-1} L_\tau + \frac{240S}{1-\gamma}\lambda^1_t -3\sum_{\tau=1}^t \lambda^\tau_t \Theta_{\tau} + \frac{87S\eta\err}{(1-\gamma)^2}\sum_{\tau=1}^t \lambda^\tau_t \nonumber \\
    &\leq \frac{1}{1+\eta^2C'^2}\sum_{\tau=1}^t \lambda^\tau_t (L_\tau+4.5\Theta_\tau) 
    + \frac{3C_2 S}{1-\gamma}\sum_{\tau=1}^{t-1} \lambda^\tau_t \alpha_{\tau-1} L_\tau  + \frac{240S}{1-\gamma}\lambda^1_t   + \frac{87S\eta\err}{(1-\gamma)^2}\sum_{\tau=1}^t \lambda^\tau_t    \tag{by our choice of $\eta$,  $\frac{6C_1S\eta^2}{(1-\gamma)^5}\leq 3$ }   \nonumber 
\end{align}
\endgroup
where in the second-to-last inequality we use \pref{lem: consecutive lambda for special choice}: with the special choice of $\alpha_t$ specified in the theorem, we have $\lambda^{\tau}_t = \alpha_t \leq \lambda^{\tau+1}_t$ for $\tau\leq t-1$.

Let $t_0=\min\left\{\tau: \frac{3C_2S}{1-\gamma}\alpha_{\tau}\leq \frac{\eta^2 C'^2}{2}\right\}$. Then we have 
\begin{align*}
    &\sum_{\tau=2}^{t+1} \lambda^{\tau}_{t+1}(L_\tau+4.5\Theta_\tau) \\
    &\leq \left(\frac{1}{1+\eta^2C'^2}+\frac{\eta^2C'^2}{2}\right)\sum_{\tau=1}^t \lambda^\tau_t (L_\tau+4.5\Theta_\tau) 
    + \frac{3C_2 S}{1-\gamma}\sum_{\tau=1}^{\min\{t_0, t\}} \lambda^\tau_t \alpha_{\tau-1} L_\tau + \frac{240S}{1-\gamma}\lambda^1_t +\frac{87S\eta\err}{(1-\gamma)^2}\sum_{\tau=1}^t \lambda^\tau_t  \\
    &\leq \frac{1}{1+0.1\eta^2C'^2}\sum_{\tau=3}^t \lambda^\tau_t (L_\tau+4.5\Theta_\tau) 
    + \frac{12C_2 S^2}{1-\gamma}\sum_{\tau=1}^{\min\{t_0, t\}} \lambda^\tau_t \alpha_{\tau-1}  + \frac{240S}{1-\gamma}\lambda^1_t + \frac{87S\eta\err}{(1-\gamma)^2}\sum_{\tau=1}^t \lambda^\tau_t.    \tag{$\eta C'\leq 2^{-15}$ according to \pref{lemma: using SPRSI}, $L_\tau\leq S\cdot \max_{z,z'}\|z-z'\|^2\leq 4S$}
\end{align*}
Finally, we add $\lambda^1_{t+1}(L_1+4.5\Theta_1)$ to both sides, and note that 
\begin{align*}
    \lambda^1_{t+1}(L_1+4.5\Theta_1) = \alpha_{t+1}(L_1+4.5\Theta_1)\leq \alpha_t(L_1+4.5\Theta_1)\leq \alpha_t\cdot 22S=22S\lambda_t^1,    
\end{align*}
where the first and second equality is by \pref{lem: consecutive lambda for special choice}. Then we get 
\begin{align*}
    &\sum_{\tau=1}^{t+1} \lambda^{\tau}_{t+1}(L_\tau+4.5\Theta_\tau) \\
    &\leq \frac{1}{1+0.1\eta^2C'^2}\sum_{\tau=1}^t \lambda^\tau_t (L_\tau+4.5\Theta_\tau) 
    + \frac{12C_2 S^2}{1-\gamma}\sum_{\tau=1}^{\min\{t_0, t\}} \lambda^\tau_t \alpha_{\tau-1}  + \frac{240S}{1-\gamma}\lambda^1_t + 22S\lambda^1_t + \frac{87S\eta\err}{(1-\gamma)^2}\sum_{\tau=1}^t \lambda^\tau_t \\
    &\leq \frac{1}{1+0.1\eta^2C'^2}\sum_{\tau=1}^t \lambda^\tau_t (L_\tau+4.5\Theta_\tau) 
    + \frac{274C_2 S^2}{1-\gamma}\sum_{\tau=1}^{\min\{t_0, t\}} \lambda^\tau_t \alpha_{\tau-1} + \frac{87S\eta\err}{(1-\gamma)^2}\sum_{\tau=1}^t \lambda^\tau_t. 
\end{align*}
Define
\begin{align*}
    Y_t \triangleq \sum_{\tau=1}^t \lambda^{\tau}_t (L_\tau+4.5\Theta_\tau).  
\end{align*}
Then we can further write that for $t\geq 1$, 
\begin{align*}
    Y_{t+1}\leq \frac{1}{1+0.1\eta^2C'^2}Y_t + \frac{274C_2S^2 t_0}{1-\gamma}\lambda^{\min\{t_0, t\}}_t + \frac{87S\eta\err}{(1-\gamma)^2}\sum_{\tau=1}^t \lambda^\tau_t.  \tag{upper bounding $\alpha_{\tau-1}$ by $1$}
\end{align*} 
Applying \pref{lemma: final recursion} with $c=\frac{0.1\eta^2C'^2}{1+0.1\eta^2C'^2}$, $g_t=Y_{t+1}$, $h_t=\frac{274C_2S^2t_0}{1-\gamma}\lambda_t^{\min\{t_0,t\}}+\frac{87S\eta\err}{(1-\gamma)^2}\sum_{\tau=1}^t\lambda_t^{\tau}$, we get 
\begin{align*} 
    Y_t 
    &\leq Y_1(1+0.1\eta^2 C'^2)^{-t} + \frac{2}{0.1\eta^2 C'^2} \left(\frac{274C_2S^2t_0}{1-\gamma} + \frac{87S\eta\err}{(1-\gamma)^2} \sup_{t'\in[1,\frac{t}{2}]} \sum_{\tau=1}^{t'} \lambda^\tau_t \right)(1+0.1\eta^2 C'^2)^{-\frac{t}{2}} \\
    & \qquad \qquad + \frac{2}{0.1\eta^2 C'^2} \left(\frac{274C_2S^2t_0}{1-\gamma}\sup_{t'\in[ \frac{t}{2},t]}\lambda_{t'}^{\min\{t_0, t'\}} +  \frac{87S\eta\err}{(1-\gamma)^2} \sup_{t'\in[\frac{t}{2},t]} \sum_{\tau=1}^{t'} \lambda^\tau_t \right) \tag{$\lambda_t^{\tau} \leq 1$}. 
\end{align*}
With the choice of $\alpha_t=\frac{H+1}{H+t}$ where $H=\frac{2}{1-\gamma}$, we have 
\begin{align*}
    t_0 &=\Theta\left(\frac{6C_2 S}{(1-\gamma)\eta^2C'^2}(H+1)\right)=\Theta\left(\frac{S}{(1-\gamma)^2\eta^2C'^2}\right)\\
    \sup_{t'\in[1, t]} \sum_{\tau=1}^{t'}\lambda^\tau_t &\leq \sup_{t'\in[1, t]} t'\alpha_t \leq \frac{t(H+1)}{H+t}\leq H+1 = \order\left(\frac{1}{1-\gamma}\right)  \tag{\pref{lem: consecutive lambda for special choice}}  \\
    \sup_{t'\in[ \frac{t}{2},t]}\lambda_{t'}^{\min\{t_0, t'\}}    \tag{\pref{lem: consecutive lambda for special choice}}
    &=\begin{cases}
        1  &\text{if $\frac{t}{2}\leq t_0$} \\
        \alpha_{\frac{t}{2}} &\text{else}  
    \end{cases}
\end{align*}
Combining them and noticing that $(1+0.1\eta^2C'^2)^{-\frac{t}{2}}= \order(\frac{1}{t})$ when $t\geq \frac{20}{\eta^2C'^2}$, we get that for $t\geq 2t_0=\Theta\left( \frac{S}{(1-\gamma)^2\eta^2C'^2} \right)$, 
\begin{align*}
    Y_t = \order\left( \frac{S^3}{\eta^4C'^4(1-\gamma)^3 }\alpha_t + \frac{S\err }{\eta C'^2(1-\gamma)^3} \right) = \order\left( \frac{S^3}{\eta^4C'^4(1-\gamma)^4 t} + \frac{S\err }{\eta C'^2(1-\gamma)^3} \right). 
\end{align*}
Since $Y_t\leq 22S+22S(t-1)\alpha_t\leq \order(\frac{S\log t}{1-\gamma})$, the above bound also trivially holds for $t\leq 2t_0$. 
Then noticing that $L_t \leq Y_t$ finishes the proof. 
%
\end{proof}

\begin{lemma}\label{lem: relate duality gap}
    For any policy pair $x, y$, the duality gap on the game can be related to the duality gap on individual states as follows: 
    \begin{align*}
        \max_{s, x', y'} \left(V^s_{x,y'} - V^s_{x', y}\right) \leq \frac{2}{1-\gamma} \max_{s,x',y'}\left(x^sQ_\star^sy'^s - x'^s Q_\star^sy^s\right).
    \end{align*}
\end{lemma}
\begin{proof}
       Notice that for any policy $x$ and state $s$, 
    \begin{align*}
        \max_{y'} V_{x, y'}^s - V_{\star}^s 
        &= \sum_{a,b}x^s(a)y'^s(b)Q^s_{x, y'}(a,b) - \sum_{a,b} x^s_\star(a)y^s_\star(b)Q_\star^s(a,b)\\
        &= \sum_{a,b}x^s(a)y'^s(b)\left(Q^s_{x, y'}(a,b) - Q^s_\star(a,b)\right) + \sum_{a,b}\left(x^s(a)y'^s(b) -  x^s_\star(a)y^s_\star(b)\right)Q_\star^s(a,b) \\
        &= \gamma\sum_{a,b}x^s(a)y'^s(b)p(s'|s,a,b)\left(V^{s'}_{x, y'} - V^{s'}_\star\right) + x^s Q_\star^s y'^s - x_\star^s Q_\star^s y_\star^s \\
        &\leq \gamma \max_{s', y'}\left(V^{s'}_{x, y'} - V^{s'}_\star\right) + x^s Q_\star^s y'^s - x_\star^s Q_\star^s y_\star^s. 
    \end{align*}
    Taking max over $s$ on two sides and rearranging, we get  
    \begin{align*}
        \max_{s, y'}\left(V^{s}_{x, y'} - V^{s}_\star\right) \leq \frac{1}{1-\gamma}\max_{s,y'}\left(x^s Q_\star^s y'^s - x_\star^s Q_\star^s y_\star^s\right) \leq \frac{1}{1-\gamma}\max_{s,x',y'}\left(x^s Q_\star^s y'^s - x'^s Q_\star^s y^s\right). 
    \end{align*}
    Similarly, 
    \begin{align*}
        \max_{s, x'}\left( V^{s}_\star - V^{s}_{x', y}\right) \leq \frac{1}{1-\gamma}\max_{s,x',y'}\left(x^s Q_\star^s y'^s - x'^s Q_\star^s y^s\right). 
    \end{align*}
    Combining the two inequalities, we get 
    \begin{align*}
        \max_{s, x', y'} \left(V^s_{x,y'} - V^s_{x', y}\right) \leq \frac{2}{1-\gamma} \max_{s,x',y'}\left(x^sQ_\star^sy'^s - x'^s Q_\star^sy^s\right). 
    \end{align*}
\end{proof}

\section{Auxiliary Lemmas}
\label{app: aux lemmas}

\subsection{Interactions between the Auxiliary Coefficients}\label{app: interaction}

\begin{lemma}\label{lemma: double recursion}
    Let $\{g_t\}_{t=1,2,\ldots}, \{h_t\}_{t=1,2,\ldots}$ be non-negative sequences that satisfy
    $g_{t} \leq \gamma\sum_{\tau=1}^{t-1} \alpha^{\tau}_{t-1} g_\tau + h_t$ for all $t\geq 1$. Then 
    $g_t \leq \sum_{\tau=1}^t \beta^\tau_t h_\tau$. 
\end{lemma}

\begin{proof}
    We prove it by induction. When $t=1$, the condition guarantees $g_1\leq h_1 = \beta^1_1 h_1$. Suppose that it holds for $1, \ldots, t-1$.  Then
    \begin{align*}
        g_{t}&\leq \gamma\sum_{\tau=1}^{t-1} \alpha_{t-1}^\tau g_\tau + h_{t}  \\
        &\leq \gamma\sum_{\tau=1}^{t-1} \alpha^\tau_{t-1} \left(\sum_{i=1}^{\tau} \beta^i_\tau h_i\right) + h_t \\
        &= \sum_{i=1}^{t-1}\left(\sum_{\tau=i}^{t-1}\gamma\alpha^\tau_{t-1} \beta^i_\tau \right)h_i + h_t
    \end{align*}
    It remains to prove that $\sum_{\tau=i}^{t-1}\gamma\alpha^\tau_{t-1} \beta^i_\tau \leq \beta^i_t$ for all $i\leq t-1$. We use another induction to show this. Fix $i$ and $t$, and define the partial sum $\zeta_{r}=\sum_{\tau=i}^r \gamma\alpha^\tau_{t-1} \beta^i_\tau$ for $r\in[i, t-1]$. Below we show that 
    \begin{align}
        \zeta_r\leq \alpha_i \prod_{\tau=i}^r (1-\alpha_\tau+\alpha_\tau \gamma) \prod_{\tau=r+1}^{t-1}(1-\alpha_{\tau}).    \label{eq: double induction}
    \end{align}
    Notice that the right-hand side above is $\beta^i_t$ when $r=t-1$, which is exactly what we want to prove. 
    
    When $r=i$, $\zeta_r=\gamma\alpha^{i}_{t-1} = \gamma\alpha_i \prod_{\tau=i+1}^{t-1}(1-\alpha_\tau)\leq \alpha_i (1-\alpha_i + \alpha_i \gamma)\prod_{\tau=i+1}^{t-1}(1-\alpha_\tau)$ where the inequality is because $1-\alpha_i + \alpha_i\gamma -\gamma=(1-\alpha_i)(1-\gamma)\geq 0$.  Now assume that \pref{eq: double induction} holds up to $r$ for some $r\geq i$. Then 
    \begin{align*}
        \zeta_{r+1} 
        &= \zeta_r + \beta^{i}_{r+1}\gamma\alpha^{r+1}_{t-1} \\
        &\leq \alpha_i \prod_{\tau=i}^r(1-\alpha_\tau+\alpha_\tau \gamma)\prod_{\tau=r+1}^{t-1}(1-\alpha_\tau) +  \left(\alpha_i\prod_{\tau=i}^r(1-\alpha_\tau + \alpha_\tau \gamma)\right)\gamma\alpha_{r+1}\prod_{\tau=r+2}^{t-1}(1-\alpha_\tau)   \\
        &= \alpha_i  \prod_{\tau=i}^{r+1}(1-\alpha_\tau+\alpha_\tau \gamma)\prod_{\tau=r+2}^{t-1}(1-\alpha_\tau).  
    \end{align*}
    This finishes the induction. 
\end{proof}

\begin{lemma}\label{lemma: triple recursion}
    Let $\{h_t\}_{t=1,2,\ldots}$ and $\{k_t\}_{t=1,2,\ldots}$ be non-negative sequences that satisfy $h_t = \sum_{\tau=1}^t \alpha^\tau_t k_\tau$. 
    Then 
    $
        \sum_{\tau=2}^t \beta^\tau_t h_{\tau-1} \leq \frac{1}{\gamma^2}\sum_{\tau=1}^{t-1} \beta^\tau_t k_\tau.  
    $
\end{lemma}

\begin{proof} 
    By the assumption on $h_\tau$, we have
    \begin{align*} 
        \sum_{\tau=2}^t \beta^\tau_t h_{\tau-1} &\leq \sum_{\tau=2}^t \beta^\tau_t \left(\sum_{i=1}^{\tau-1}\alpha^i_{\tau-1} k_i \right)
        = \sum_{i=1}^{t-1} \left(\sum_{\tau=i+1}^{t} \beta^\tau_t \alpha^i_{\tau-1}\right) k_i
    \end{align*}
    It remains to prove that for $i<t$, $\sum_{\tau=i+1}^t \beta^\tau_t \alpha^i_{\tau-1}\leq \frac{1}{\gamma^2}\beta^i_t$, or equivalently, $\gamma\sum_{\tau=i+1}^t \beta^\tau_t \alpha^i_{\tau-1}\leq \frac{1}{\gamma}\beta^i_t$. Below we use another induction to prove this. Fix $i$ and $t$, and define the partial sum $\zeta_r = \gamma\sum_{\tau=r}^t \beta^\tau_t \alpha^i_{\tau-1}$ for $r\in[i+1, t]$. We will show that 
    \begin{align}
        \zeta_r \leq \alpha_i \prod_{\tau=i+1}^{r-1}(1-\alpha_\tau)\prod_{\tau=r}^{t-1}(1-\alpha_\tau + \alpha_\tau \gamma).   \label{eq: recur hard}
    \end{align}
    For the base case $r=t$, we have 
    \begin{align*}
        \zeta_r=\gamma\beta^t_t\alpha^i_{t-1}=\gamma\alpha_i\prod_{\tau=i+1}^{t-1}(1-\alpha_\tau)\leq \alpha_i\prod_{\tau=i+1}^{t-1}(1-\alpha_\tau).
    \end{align*}
    Suppose that \pref{eq: recur hard} holds up to $r$ for some $r\leq t$. Then
    \begin{align*}
        \zeta_{r-1} 
        &= \zeta_r + \alpha^i_{r-2} \gamma\beta^{r-1}_t  \\
        &\leq \alpha_i \prod_{\tau=i+1}^{r-1}(1-\alpha_\tau)\prod_{\tau=r}^{t-1}(1-\alpha_\tau + \alpha_\tau \gamma) + \alpha_i \left(\prod_{\tau=i+1}^{r-2}(1-\alpha_\tau) \right) \gamma\alpha_{r-1}\prod_{\tau=r-1}^{t-1}(1-\alpha_\tau+\alpha_\tau \gamma) \\
        &\leq \left(\alpha_i \prod_{\tau=i+1}^{r-2}(1-\alpha_\tau)\prod_{\tau=r}^{t-1}(1-\alpha_\tau + \alpha_\tau \gamma)\right) \times (1-\alpha_{r-1} + \gamma\alpha_{r-1}(1-\alpha_{r-1}+\alpha_{r-1}\gamma)) \\
        &\leq \left(\alpha_i \prod_{\tau=i+1}^{r-2}(1-\alpha_\tau)\prod_{\tau=r}^{t-1}(1-\alpha_\tau + \alpha_\tau \gamma)\right) \times (1-\alpha_{r-1} + \alpha_{r-1}\gamma) \\
        &= \alpha_i \prod_{\tau=i+1}^{r-2}(1-\alpha_\tau)\prod_{\tau=r-1}^{t-1}(1-\alpha_\tau + \alpha_\tau \gamma).   
    \end{align*}
    This finishes the induction. Applying the result with $r=i+1$, we get
    \begin{align*}
        \zeta_{i+1}\leq\alpha_i\prod_{\tau=i+1}^{t-1}(1-\alpha_\tau +\alpha_\tau \gamma) = \frac{\beta^i_t}{1-\alpha_i + \alpha_i \gamma}\leq \frac{\beta^i_t}{\gamma} 
    \end{align*}
    where the last inequality is by $1-\alpha_i + \alpha_i \gamma - \gamma = (1-\alpha_i)(1-\gamma)\geq 0$. This finishes the proof.  
\end{proof}

\begin{lemma}\label{lemma: delta formula}
    For $0\leq h\leq t$, 
    $\sum_{\tau=0}^{h}\alpha^\tau_{t} = \delta^h_t$. 
\end{lemma}
\begin{proof} 
    We prove it by induction on $h$. When $h=0$, $\sum_{\tau=0}^{h}\alpha^\tau_{t}=\alpha^0_{t} =  \prod_{\tau=1}^{t}  (1-\alpha_\tau)=\delta^h_t$ since $\alpha_0=1$. Suppose that the formula holds for $h$. Then $\sum_{\tau=0}^{h+1}\alpha^\tau_{t} = \sum_{\tau=0}^{h}\alpha^\tau_{t} + \alpha^{h+1}_{t} = \prod_{\tau=h+1}^{t}(1-\alpha_\tau) + \alpha_{h+1} \prod_{\tau=h+2}^{t}(1-\alpha_\tau) = \prod_{\tau=h+2}^{t}(1-\alpha_\tau)=\delta^{h+1}_t$, which finishes the induction.   
\end{proof}

\begin{lemma} \label{lem: lambda beta lambda}
For any positive integers $i,t$ with $i\leq t$, 
    $\sum_{\tau=i}^t \lambda^\tau_t \beta^i_\tau \leq \frac{3}{1-\gamma}\lambda^i_t$. 
\end{lemma}
\begin{proof}
    Notice that 
\begin{align}
    \sum_{\tau=i}^t \lambda^{\tau}_t\beta^i_\tau = \lambda^i_t  + \sum_{\tau=i+1}^{t} \lambda^{\tau}_t\beta^i_\tau.
\end{align}
Below we use induction to prove
\begin{align*}
    \sum_{\tau=r}^{t}\lambda^\tau_t \beta^i_\tau \leq \frac{2}{1-\gamma} \alpha_i \prod_{\tau=i}^{r-1}\Big(1-\alpha_\tau(1-\gamma)\Big)\prod_{\tau=r}^{t-1}\lambda_\tau
\end{align*}
for $r\in [i+1, t]$. 
When $r=t$, $\sum_{\tau=r}^{t}\lambda^\tau_t \beta^i_\tau = \lambda^{t}_t\beta^{i}_{t}= \beta^i_t = \alpha_i \prod_{\tau=i}^{t-1}(1-\alpha_\tau(1-\gamma))$. Suppose this holds for some $r\leq t$. Then
\begin{align*}
    \sum_{\tau=r-1}^t  \lambda^\tau_t \beta^i_\tau
    &\leq \frac{2}{1-\gamma} \alpha_i \prod_{\tau=i}^{r-1}\Big(1-\alpha_\tau(1-\gamma)\Big)\prod_{\tau=r}^{t-1}\lambda_\tau + \beta^i_{r-1}\lambda^{r-1}_t \\
    &\leq \frac{2}{1-\gamma} \alpha_i \prod_{\tau=i}^{r-1}\Big(1-\alpha_\tau(1-\gamma)\Big)\prod_{\tau=r}^{t-1}\lambda_\tau 
    + \left(\alpha_i \prod_{\tau=i}^{r-2}\Big(1-\alpha_\tau(1-\gamma)\Big)\right)\alpha_{r-1}\prod_{\tau=r-1}^{t-1}\lambda_\tau \\
    &\leq \left[\alpha_i \prod_{\tau=i}^{r-2}\Big(1-\alpha_\tau(1-\gamma)\Big) \prod_{\tau=r}^{t-1}\lambda_\tau\right] \left(\frac{2}{1-\gamma}\left(1-\alpha_{r-1}(1-\gamma)\right) + \alpha_{r-1} \right)  \tag{$\lambda_{r-1}\leq 1$}\\
    &= \left[\alpha_i \prod_{\tau=i}^{r-2}\Big(1-\alpha_\tau(1-\gamma)\Big) \prod_{\tau=r}^{t-1}\lambda_\tau \right] \left(\frac{2}{1-\gamma}\left(1-\frac{1}{2}\alpha_{r-1}(1-\gamma)\right) \right) \\
    &\leq \frac{2}{1-\gamma}\alpha_i \prod_{\tau=i}^{r-2}\Big(1-\alpha_\tau(1-\gamma)\Big) \prod_{\tau=r-1}^{t-1}\lambda_\tau,    \tag{$\lambda_{r-1}\geq 1-\frac{1}{2}\alpha_{r-1}(1-\gamma)$ by the definition of $\lambda_{r-1}$}
\end{align*}
which finishes the induction. 
Notice that this implies 
\begin{align*}
\sum_{\tau=i+1}^t \lambda^\tau_t \beta^i_\tau \leq \frac{2}{1-\gamma}\alpha_i \left(1-\alpha_i(1-\gamma)\right) \prod_{\tau=i+1}^{t-1}\lambda_\tau \leq 
\frac{2}{1-\gamma}\alpha_i \prod_{\tau=i}^{t-1}\lambda_\tau =
\frac{2}{1-\gamma}\lambda^i_t
\end{align*}
where the second inequality is by the definition of $\lambda_i$. 
Thus, 
\begin{align} 
    \sum_{\tau=i}^t \lambda^\tau_t \beta^i_\tau \leq \frac{3}{1-\gamma}\lambda^i_t. \label{eq: cal coeff}
\end{align}
\end{proof}

\begin{lemma}
    \label{lem: lambda ratio consecutive}
    $\lambda^{\tau+1}_{t+1}\leq \lambda^\tau_t$. 
\end{lemma}
\begin{proof}
    When $\tau < t$, we have 
\begin{align*}
    \frac{\lambda^{\tau+1}_{t+1}}{\lambda^\tau_t}
    &= \frac{\alpha_{\tau+1} \Pi_{i=\tau+1}^{t}\lambda_i}{\alpha_\tau  \Pi_{i=\tau}^{t-1}\lambda_i} \leq \frac{\alpha_{\tau+1}}{\alpha_\tau \lambda_\tau} \leq 1
\end{align*} 
where in the first inequality we use $\lambda_t\leq 1$ and in the second inequality we use the definition of $\lambda_\tau$. 
    When $\tau=t$, we have $\frac{\lambda^{\tau+1}_{t+1}}{\lambda^\tau_t}=\frac{1}{1}=1$. 
\end{proof}

\begin{lemma}\label{lemma: sum1 of beta}
    $\sum_{\tau=1}^t \beta^\tau_t\leq \frac{2}{1-\gamma}$. 
\end{lemma}
\begin{proof}
    Below we use induction to prove that for all $r=1,2,\ldots, t-1$, 
    \begin{align*}
        \sum_{\tau=1}^{r} \beta^\tau_t \leq \frac{1}{1-\gamma}\prod_{i=r}^{t-1}(1-\alpha_i+\alpha_i \gamma). 
    \end{align*}
    When $r=1$, the left-hand side is $\beta^{1}_t = \alpha_{1}\prod_{i=1}^{t-1}(1-\alpha_i+\alpha_i \gamma) \leq \frac{1}{1-\gamma}\prod_{i=1}^{t-1}(1-\alpha_i+\alpha_i \gamma)$, which is the right-hand side. 
    
    Suppose that this holds for $r$, then
    \begin{align*}
        \sum_{\tau=1}^{r+1} \beta^\tau_t 
        &= \beta^{r+1}_{t} + \sum_{\tau=1}^{r} \beta^\tau_t  \\
        &\leq \alpha_{r+1}\prod_{i=r+1}^{t-1} (1-\alpha_i + \alpha_i\gamma) + \frac{1}{1-\gamma}\prod_{i=r}^{t-1}(1-\alpha_i+\alpha_i \gamma) \\
        &\leq  \left(\frac{1}{1-\gamma}\prod_{i=r+1}^{t-1} (1-\alpha_i + \alpha_i\gamma)\right) \left( \alpha_{r+1}(1-\gamma) + 1-\alpha_{r}(1-\gamma)  \right)  \\
        &\leq \frac{1}{1-\gamma}\prod_{i=r+1}^{t-1} (1-\alpha_i + \alpha_i\gamma)   \tag{because $\alpha_{r+1}\leq \alpha_r$}
    \end{align*}
    which finishes the induction. 
    
    Therefore, $\sum_{\tau=1}^t\beta^\tau_t = 1 + \sum_{\tau=1}^{t-1}\beta^\tau_t\leq 1+\frac{1}{1-\gamma} \leq \frac{2}{1-\gamma}$. 
\end{proof}

\begin{lemma}
    \label{lemma: final recursion}
    Let $\{g_t\}_{t=0,1,2,\ldots}$, $\{h_t\}_{t=1,2,\ldots}$ be non-negative sequences that satisfy $g_t\leq (1-c)g_{t-1} + h_t$ for some $c\in(0,1)$ for all $t\geq 1$. 
    Then 
   \begin{align*}
       g_t \leq g_0(1-c)^t + \frac{\max_{\tau\in[1,\nicefrac{t}{2}]} h_\tau}{c}(1-c)^{\frac{t}{2}} + \frac{\max_{\tau\in[\nicefrac{t}{2}, t]} h_\tau}{c}.  
   \end{align*}
\end{lemma}

\begin{proof}
    We first show that 
    \begin{align*}
        g_t \leq g_0(1-c)^t + \sum_{\tau=1}^t (1-c)^{t-\tau}h_\tau. 
    \end{align*}
    The case of $t=1$ is clear. Suppose that this holds for $g_t$. Then 
    \begin{align*}
        g_{t+1}
        &\leq (1-c)\left(g_0(1-c)^t + \sum_{\tau=1}^t (1-c)^{t-\tau}h_\tau\right) + h_{t+1} 
        = g_0 (1-c)^{t+1} + \sum_{\tau=1}^{t+1}(1-c)^{t+1-\tau}h_\tau, 
    \end{align*}
    which finishes the induction. 
    Therefore, 
    \begin{align*}
        g_t 
        &\leq g_0(1-c)^t + \sum_{\tau=1}^{\nicefrac{t}{2}}(1-c)^{t-\tau}\max_{\tau\in[1,\nicefrac{t}{2}]}h_\tau + \sum_{\tau=\nicefrac{t}{2}+1}^t (1-c)^{t-\tau}\max_{\tau\in[\nicefrac{t}{2}, t]}h_\tau \\
        &\leq g_0(1-c)^t + \frac{\max_{\tau\in[1,\nicefrac{t}{2}]} h_\tau}{c}(1-c)^{\frac{t}{2}} + \frac{\max_{\tau\in[\nicefrac{t}{2}, t]} h_\tau}{c}.  
    \end{align*}
    
\end{proof}

\subsection{\texorpdfstring{Some Properties for the choice of $\alpha_t=\frac{H+1}{H+t}$}{}}\label{app: special learning rate 1}

\begin{lemma}
    \label{lem: sum of beta 1}
    For the choice $\alpha_t=\frac{H+1}{H+t}$ with $H\geq \frac{2}{1-\gamma}$, we have 
    $\sum_{t=\tau}^\infty \beta^\tau_t \leq H+3$. 
\end{lemma} 
\begin{proof}
    When $t\geq \tau+2$, 
    \begin{align*}
        \beta^\tau_t
        &=\alpha_\tau \prod_{i=\tau}^{t-1}\left(1-\alpha_i(1-\gamma)\right) \\
        &\leq \alpha_\tau \prod_{i=\tau}^{t-1}\left(1-\alpha_i\times \frac{2}{H+1}\right)   \tag{$H+1\geq \frac{2}{1-\gamma}$} \\
        &= \alpha_\tau \prod_{i=\tau}^{t-1}\left(1-\frac{2}{H+i}\right)\\
        &= \frac{H+1}{H+\tau} \times \frac{H+\tau-2}{H+\tau}\times \frac{H+\tau-1}{H+\tau+1}\times\cdots \times \frac{H+t-3}{H+t-1}\\
        &= \frac{H+1}{H+\tau}\times \frac{(H+\tau-2)(H+\tau-1)}{(H+t-2)(H+t-1)} \\
        &= \frac{H+1}{H+\tau}(H+\tau-2)(H+\tau-1) \left(\frac{1}{H+t-2} - \frac{1}{H+t-1}\right) \\
        &\leq (H+1)(H+\tau-2)\left(\frac{1}{H+t-2} - \frac{1}{H+t-1}\right). 
    \end{align*}
    Therefore, 
    \begin{align*}
        \sum_{t=\tau+2}^{\infty} \beta^\tau_t \leq  (H+1)(H+\tau-2)\times \frac{1}{H+\tau} \leq H+1, 
    \end{align*}
    and thus $\sum_{t=\tau}^\infty \beta^\tau_t \leq H+3$. 
\end{proof}

\begin{lemma}\label{lem: consecutive lambda for special choice}
    For the choice $\alpha_t=\frac{H+1}{H+t}$ with $H\geq \frac{2}{1-\gamma}$, we have $\lambda^\tau_t=\alpha_t$ for $\tau<t$. 
\end{lemma}
\begin{proof}
    With this choice of $\alpha_t$, 
    \begin{align*}
        \lambda_t &= \max\left\{ \frac{H+t}{H+t+1}, 1-\frac{1-\gamma}{2}\times \frac{H+1}{H+t} \right\} \\
        &= \max\left\{ 1-\frac{1}{H+t+1}, 1-\frac{1-\gamma}{2}\times \frac{H+1}{H+t} \right\}.
    \end{align*}
    By the condition, we have $\frac{1-\gamma}{2}\times \frac{H+1}{H+t}\geq \frac{1}{H+t}\geq \frac{1}{H+t+1}$. Therefore, $\lambda_t = \frac{H+t}{H+t+1}=\frac{\alpha_{t+1}}{\alpha_t}$. Thus for $\tau<t$, 
    \begin{align*}
        \lambda^\tau_t = \alpha_\tau \prod_{i=\tau}^{t-1}\frac{\alpha_{i+1}}{\alpha_i} = \alpha_t. 
    \end{align*}
\end{proof}

\section{Analysis on Sample Complexity}\label{app: samples}
\subsection{Proof of \pref{thm: samples}}
\begin{proof}
As long as we can make 
\begin{align}
    \left|\ell^s_t(a) -\mathbf{e}_a^\top Q^s_t \ytil^s_t\right| \leq \frac{\err}{2|\calA|}\label{eq: epsA}
\end{align}
hold with high probability, then $\left|\ell^s_t(a) - \mathbf{e}_a^\top Q^s_ty^s_t\right|\leq \left|\ell^s_t(a) - \mathbf{e}_a^\top Q^s_t \ytil^s_t\right| + \left|\mathbf{e}_a^\top Q^s_t \ytil^s_t - \mathbf{e}_a^\top Q^s_ty^s_t\right|\leq \frac{\err}{2|\calA|} + \frac{1}{1-\gamma}\times \frac{\err'}{2|\calA|}\leq \frac{\err}{|\calA|}$, which implies $\|\ell^s_t - Q^s_ty_t^s\|\leq \err$. We can similarly ensure $\|r^s_t - Q_t^{s^\top} x^s_t\|\leq \err$ and $|\pay^s_t - x_t^{s^\top}Q^s_ty^s_t|\leq \err$ by the same way.
Let $N_{s,a}\triangleq\sum_{i=1}^\LL \one[s_i=s, a_i=a]$ be the number of times the $(s,a)$ pair is visited.
For a deterministic $N$, we can use Azuma-Hoeffding inequality and know that \pref{eq: epsA} holds with probability $1-\delta$ if $N =\Omega\left(\frac{|\calA|^2}{\err^2}\log(1/\delta)\right)$.
However, $N_{s,a}$ is random, so we cannot use Azuma-Hoeffding's inequality directly.
Let $(b^{(1)}, s^{(1)}),(b^{(2)}, s^{(2)}),\dots$ be a sequence of independent random variables where $b^{(i)}\sim \ytil_t^{s}$, $s^{(i)}\sim p(\cdot|s,a,b^{(i)})$, $i=1,2,\dots$ and define $\widetilde{\ell}_{t,m}^s=\frac{1}{m}\sum_{i=1}^m(\sigma(s,a,b^{(i)})+\gamma V_{t-1}^{s^{(i)}})$. It is direct to see that $\widetilde{\ell}_{t,m}^s$ is an unbiased estimator of $\mathbf{e}_a^\top Q_t^s\ytil_t^s$. Then by Azuma-Hoeffding's inequality, we have
\begin{align*}
    &\text{Prob}\left[\left|\ell_t^s(a)-\mathbf{e}_a^\top Q^s_t \ytil^s_t\right| \leq \order\left(\sqrt{\frac{\log(L/\delta)}{N_{s,a}}}\right)\right] \\
    &\leq \text{Prob}\left[\exists m\in [L], \left|\widetilde{\ell}_{t,m}^s-\mathbf{e}_a^\top Q^s_t \ytil^s_t\right|\leq \order\left(\sqrt{\frac{\log(L/\delta)}{m}}\right)\right]\\
    &\leq \sum_{m=1}^L \text{Prob}\left[\left|\widetilde{\ell}_{t,m}^s-\mathbf{e}_a^\top Q^s_t \ytil^s_t\right|\leq \order\left(\sqrt{\frac{\log(L/\delta)}{m}}\right)\right]\leq \delta.
\end{align*}
Therefore, with probability at least $1-\delta$, \pref{eq: epsA} holds if
\begin{align}
    N_{s,a}=\Omega\left(\frac{|\calA|^2}{\err^2}\log(\LL/\delta)\right). \label{eq: Az1}
\end{align}
Now it remains to determine $\LL$ to make \pref{eq: Az1} hold with high probability.
Note that by \pref{assum:irreducibility}, we know that $T^{s'\rightarrow s}_{\xtil_t,\ytil_t}\leq \frac{1}{\mu}$ for any $s'$. Let $\calT_{s,a}$ be the distribution of random variable which is the number of rounds between the current state-action pair $(s',a')$ and the next occurrence of $(s,a)$ under strategy $\xtil_t^s$ and $\ytil_t^s$. The mean of this distribution is $t_{s,a}\leq 1+\frac{2|\calA|}{\err'\mu}\leq \frac{3|\calA|}{\err'\mu}$. Then by Markov inequality, \[\text{Prob}\left[\text{the number of rounds before reaching $(s,a)$}\leq \frac{6|\calA|}{\err'\mu}\right]\geq \frac{1}{2}.\] Therefore, with probability at least $1-\frac{\delta}{L}$, within $\Theta(\frac{|A|}{\err'\mu}\log(L/\delta))$ rounds, we reach $(s,a)$ state-action at least pair once. Thus, \pref{eq: Az1} holds  when $\LL=\Omega\left(\frac{|\calA|^3}{\err'\mu\err^2}\log^2(\LL/\delta)\right) $ with probability $1-\delta$.
Solving $\LL$ gives $\LL=\widetilde{\Omega}\left(\frac{|\calA|^3}{(1-\gamma)\mu\err^3}\log^2({1}/{\delta})\right) $.
The cases for $r_t^s(b)$ and $\pay^s_t$ are similar. 
Finally, using a union bound on all $\calA$, $\calB$, $\calS$, and all iterations, we know that with probability $1-\delta$, the $\err$-approximations are always guaranteed if we use the estimation above and take $\LL=\widetilde{\Omega}\left(\frac{|\calA|^3+|\calB|^3}{(1-\gamma)\mu\err^3}\log^2(T/{\delta})\right) $.
\end{proof}

\subsection{Proof of \pref{col: thm2}}
\begin{proof}
    From \pref{thm: main 1}, we know that in order to show $ \frac{1}{T}\sum_{t=1}^T \max_{s,x',y'} \left(V^s_{\xp_t, y'} -V^s_{x', \yp_t} \right)\le \arx$, it is sufficient to show $\frac{|\calS|}{\eta(1-\gamma)^2}\sqrt{\frac{\log T}{T}}\le \arx$ and $\frac{|\calS|\sqrt{\err}}{\sqrt{\eta}(1-\gamma)^2} \le \arx$.
    Solving these two inequalities, we get $T=\Omega\left(\frac{|\calS|^2}{\eta^2(1-\gamma)^4 \arx^2}\log\frac{|\calS|}{\eta(1-\gamma) \arx}\right)$ and $\err=\order\left(\frac{\eta(1-\gamma)^4\arx^2}{|\calS|^2}\right)$. 
    Plugging $\err$ into $L$ in \pref{thm: samples} gives the required $L$.
\end{proof}

\subsection{Proof of \pref{col: thm1}}
\begin{proof}
    From \pref{thm: main 2}, we know that in order to show $\frac{1}{|\calS|} \sum_{s\in\calS} \dist_\star^2(\zp_T^s)\le \arx$, it is sufficient to show $\frac{|\calS|^2}{\eta^4C^4(1-\gamma)^4 T}\le \arx$ and $\frac{\err}{\eta C^2(1-\gamma)^3} \le \arx$.
    Solving these two inequalities, we get $T=\Omega\left(\frac{|\calS|^2}{\eta^4C^4(1-\gamma)^4 \arx}\right)$ and $\err=\order\left({\eta C^2(1-\gamma)^3}{\arx}\right)$. 
    Plugging $\err$ into $L$ in \pref{thm: samples} gives the required $L$.
\end{proof}


\section{Analysis on Rationalily}\label{app: rationality}
In this section, we analyze the rationality of our algorithm. 
First, we present the full pseudocode of \pref{algo: main alg single}, which is the single-player-perspective version of \pref{algo: main alg}, and then prove that \pref{algo: main alg single} achieves rationality.

\subsection{Single-Player-Perspective Version of \pref{algo: main alg}}\label{app: rationality-algo}
\begin{algorithm}[ht]
    \caption{Single-Player-Perspective Optimistic Gradient Descent/Ascent for Markov Games}
    \label{algo: main alg single}
    \textbf{Parameter}: $\gamma\in[\frac{1}{2}, 1), \eta \leq \frac{1}{10^4}\sqrt{\frac{(1-\gamma)^5}{S}} $, $\err\in\left[0, \frac{1}{1-\gamma}\right]$ \\ 
    \textbf{Parameters}: a non-increasing sequence $\{\alpha_t\}_{t=1}^T$ that goes to zero.  \\ 
    \textbf{Initialization}: arbitrarily initialize $\xp_1^s=\x_1^s \in \Delta_{\calA} $ for all $s\in\calS$. \\
    $V_0^s\leftarrow 0$ for all $s\in\calS$.\\
    \For{$t=1,\ldots, T$}{
        For all $s$, 
        define 
        $\Q_t^s\in\mathbb{R}^{|\calA|\times|\calB|}$ as
        \begin{align*}
        Q_t^s(a,b) &\triangleq  
           \sig(s,a,b) + \gamma \E_{s'\sim p(\cdot|s,a,b)} \left[V_{t-1}^{s'} \right],
        \end{align*} 
        and update  
        \begin{align} 
            \xp_{t+1}^s &= \Pi_{\Delta_\calA}\Big\{\xp_t^s - \eta \ell_t^s \Big\},    \label{eq: update 1 single}\\
            \x_{t+1}^s &= \Pi_{\Delta_\calA}\Big\{\xp_{t+1}^s - \eta \ell_t^s \Big\},  \label{eq: update 2 single}\\
            V^s_{t} &= (1-\alpha_t)V^s_{t-1} + \alpha_t \pay_t^s,  \label{eq: update 5 single}
        \end{align}
        where $\ell^s_t$ and $\pay^s_t$ are $\err$-approximations of $Q^s_t y^s$ and $x_t^{s^\top} Q^s_t y^s$, respectively. 
    }
\end{algorithm}

\subsection{Analysis of \pref{algo: main alg single}}\label{app: rationality-algo-analysis}
In this section, we prove \pref{thm: rational-1}, which shows the rationality of \pref{algo: main alg single}. 
We call the original game \textbf{Game 1} and construct a two-player Markov game \textbf{Game 2} with the difference being that Player 2 has only one single action (call it $1$) on each state, 
the loss function is redefined as $\gametwo{\sig}(s,a,1)=\E_{b\sim y^s}[\sig(s,a,b)]$, and the transition kernel is redefined as $\gametwo{p}(s'|s,a,1)=\E_{b\sim y^s}[p(s'|s,a,b)]$. 
Correspondingly, we define
\begin{align*}
    &\gametwo{Q}_t^s(a,1) = \gametwo{\sigma}(s,a,1) + \gamma\mathbb{E}_{s'\sim \gametwo{p}(\cdot|s,a,1)}\left[\gametwo{V}_{t-1}^{s'}\right], \\
    &\gametwo{\xp}_{t+1}^s=\Pi_{\Delta_{\calA}}\left\{\gametwo{\xp}_t^s-\eta\ell_t^s\right\}, \\
    &\gametwo{x}_{t+1}^s=\Pi_{\Delta_{\calA}}\left\{\gametwo{\xp}_{t+1}^s-\eta\ell_t^s\right\}, \\
    &\gametwo{V}^{s}_{t}=(1-\alpha_t)\gametwo{V}_{t-1}^s + \alpha_t\rho_t^s,
\end{align*}
where $\gametwo{V}^{s}_{0} = 0$ for all $s$, and $\ell_t^s$ and $\rho_t^s$ are the same as in \pref{algo: main alg single}.
Clearly, the sequences $\{\xp_t, \x_t\}_{t\in [T]}$ and $\{\gametwo{\xp}_t, \gametwo{\x}_t\}_{t\in [T]}$ are exactly the same (assuming their initializations are the same, that is, $\xp_1=\gametwo{\xp}_1$ and $x_1=\gametwo{x}_1$).
In the following lemma, we show that $\ell_{t}^s$ and $\rho_{t}^s$ are indeed $\err$-approximation of $\gametwo{Q}_{t}^s(\cdot,1)$ and $\gametwo{x}_{t}^{s^\top}\gametwo{Q}_t^s(a,\cdot)$,
which then implies that the sequence $\{\gametwo{\xp}_t\}_{t\in [T]}$ is indeed the output of our \pref{algo: main alg} for \textbf{Game 2}
(note that we can think of Player 2 executing \pref{algo: main alg} in \textbf{Game 2} as well since she only has one unique strategy).
Realizing that $\calX_{\star}^s$ for \textbf{Game 2} is exactly the set of best responses of $y^s$, we can thus conclude that \pref{thm: rational-1} is a direct corollary of \pref{thm: main 1} and \pref{thm: main 2}.

\begin{lemma}\label{lem: rational-matching}
For all $t$ and $s$, $\ell_{t}^s$ and $\rho_{t}^s$ are $\err$-approximation of $\gametwo{Q}_{t}^s(\cdot,1)$ and $\gametwo{x}_{t}^{s^\top}\gametwo{Q}_t^s(a,\cdot)$ respectively.
\end{lemma}



\begin{proof}
    We prove the result together with $\gametwo{V}_t^s=V_t^s$, $\gametwo{Q}_t^s(\cdot,1)=Q_t^sy^s$ for all $t\in [T]$ by induction.
    When $t=1$, $\gametwo{V}_t^s=V_t^s$ clearly holds. In addition, $Q_1^s(a,\cdot)y^s=\mathbb{E}_{b\sim y^s}[\sigma(s,a,b)]=\gametwo{Q}_1^s(a,1)$.
    Therefore $\ell_1^s$ and $\rho_1^s$ are indeed $\err$-approximation of $\gametwo{Q}_1^s(a,\cdot)$ and $\gametwo{x}_1^{s^\top}\gametwo{Q}_t^s(a,\cdot)$.
    
    Suppose that the claim holds at $t$.
    By definition and the inductive assumption, we have
\begin{align*}
    \gametwo{Q}_{t+1}^s(a,1)&=\gametwo{\sigma}(s,a,1)+\gamma\mathbb{E}_{s'\sim \gametwo{p}(\cdot|s,a,1)}\left[\gametwo{V}_{t}^{s'}\right] \\
    &= \mathbb{E}_{b\sim y^s}\left[\sigma(s,a,b)+\gamma\mathbb{E}_{s'\sim p(s'|s,a,b)}\left[V_{t}^{s'}\right]\right] \\
    &=Q_{t+1}^s(a,\cdot)y^s,
\end{align*}
which also shows that $\ell_{t+1}^s$ and $\rho_{t+1}^s$ are indeed $\err$-approximation of $\gametwo{Q}_{t+1}^s(\cdot,1)$ and $\gametwo{x}_{t+1}^{s^\top}\gametwo{Q}_t^s(a,\cdot)$ (recall $\gametwo{x}_{t+1}^{s} = x_{t+1}^{s}$).
By definition of $\gametwo{V}_{t+1}^s$, we also have $\gametwo{V}_{t+1}^s=(1-\alpha_{t+1})\gametwo{V}_t^s+\alpha_t\rho_t^s=(1-\alpha_t)V_t^s+\alpha_t\rho_t^s=V_{t+1}^s$, which finishes the induction.
\end{proof}

\end{document}